\documentclass[a4paper,11pt]{article}
\usepackage[top=3.25cm, bottom=3.25cm, left=2.5cm, right=2.5cm]{geometry}
\usepackage[utf8]{inputenc} 
\usepackage[T1]{fontenc}    
\usepackage{url}            
\usepackage{amsfonts}      
\usepackage{nicefrac}       
\usepackage{microtype}      
\usepackage{times}
\usepackage[english]{babel}
\usepackage[babel]{csquotes}
\usepackage{guit}
\usepackage{mdframed}
\usepackage{morefloats}
\usepackage{etex}
\usepackage{sidecap}
\usepackage{lipsum}
\usepackage{amsmath,amssymb,amsthm} 
\usepackage{mathtools} 
\usepackage{braket} 
\usepackage{bm} 
\usepackage{empheq}        
\usepackage{booktabs}                      
\usepackage{tabularx}                       
\usepackage{graphicx}                      
\usepackage{subfig}                          
\usepackage{caption}                      
\usepackage{xcolor} 
\usepackage{subfig}        
\usepackage{ucs}
\usepackage[normalem]{ulem}
\usepackage{multirow}
\usepackage{algorithm}
\usepackage{algpseudocode}
\usepackage{amsfonts}
\usepackage{multicol}
\usepackage{ulem}
\usepackage{appendix}
\usepackage{enumitem}
\usepackage[colorlinks,hyperindex]{hyperref}
\usepackage[font=small,skip=0pt]{caption}
\usepackage{thmtools,thm-restate}

\AfterEndEnvironment{restatable}{\noindent\ignorespaces}

\definecolor{ao(english)}{rgb}{0.0, 0.5, 0.0}

\newcommand{\X}{\mathcal{X}}
\newcommand{\Y}{\mathcal{Y}}
\newcommand{\Z}{\mathcal{Z}}

\newcommand{\E}{\mathcal{E}}

\newcommand{\EE}{\mathbb{E}}

\renewcommand{\H}{\Real^d}

\newcommand{\trans}{^{\scriptscriptstyle \top}}
\newcommand{\R}{\mathbb{R}}

\newcommand{\N}{\mathbb{N}}

\newcommand{\la}{\lambda}

\newcommand{\msf}[1]{\mathsf{#1}}

\newcommand{\prox}{\ensuremath{\text{\rm prox}}}

\newcommand{\dom}{\ensuremath{\text{\rm Dom}}}

\newcommand{\argmin}{\operatornamewithlimits{argmin}}

\newcommand{\wmu}{{w_\mu}}

\newcommand{\wbar}{{\msf{m}}}

\newcommand{\wh}{{w_h}}
\newcommand{\uh}{{u_h}}
\newcommand{\whD}{{w_{h}}}
\newcommand{\bwhD}{{{\bar w}_{h}}}

\newcommand{\uhD}{{u_{h}}}

\newcommand{\hn}{{h_n}}

\newcommand{\ee}{{\mathcal{E}}}
\newcommand{\een}{{\ee_n}}

\newcommand{\eeinf}{{\ee_{\env}}}
\newcommand{\rx}{{R}}

\newcommand{\task}{\mu}
\newcommand{\env}{\rho}
\newcommand{\Zn}{Z_n}
\newcommand{\Xn}{X_n}
\newcommand{\yn}{\textbf{y}_n}

\newcommand{\LL}{\mathcal{L}}

\newcommand{\datan}{Z_n}

\newcommand{\cR}{\mathcal{R}}
\newcommand{\Real}{\mathbb{R}}
\newcommand{\Exp}{\mathbb{E}}

\newcommand{\psd}{\mathbb{S}_+^d}

\newcommand{\prhD}{{\Phi_{h}}}
\newcommand{\duhD}{{\Psi_{h}}}

\declaretheorem[name=Theorem,refname=Thm.]{theorem}

\declaretheorem[name=Lemma,sibling=theorem]{lemma}
\declaretheorem[name=Proposition,refname=Prop.,sibling=theorem]{proposition}
\declaretheorem[name=Remark]{remark}
\declaretheorem[name=Corollary,refname=Cor.,sibling=theorem]{corollary}
\declaretheorem[name=Definition,refname=Def.,sibling=theorem]{definition}

\declaretheorem[name=Assumption,refname=Asm.]{assumption}
\declaretheorem[name=Example]{example}

\makeatletter
\renewenvironment{proof}[1][\proofname]{\par
  \pushQED{\qed}%
  \normalfont \topsep6\p@\@plus6\p@\relax
  \trivlist
  \item[\hskip\labelsep
        \bfseries
    #1\@addpunct{.}]\ignorespaces
}{%
  \popQED\endtrivlist\@endpefalse
}
\makeatother
\def\eop{$\rule{1.3ex}{1.3ex}$}
\renewcommand\qedsymbol\eop

\newcommand*{\shortautoref}[1]{%
  \begingroup
    \def\sectionautorefname{Sec.}%
    \autoref{#1}%
  \endgroup
}

\newcommand*{\shortautorefsubapp}[1]{%
  \begingroup
    \def\subsectionautorefname{App.}%
    \autoref{#1}%
  \endgroup
}

\newcommand*{\shortautorefrem}[1]{%
  \begingroup
    \def\remarkautorefname{Rem.}%
    \autoref{#1}%
  \endgroup
}

\newcommand*{\shortautorefex}[1]{%
  \begingroup
    \def\exampleautorefname{Ex.}%
    \autoref{#1}%
  \endgroup
}

\title{\sffamily\huge\bf Learning-to-Learn Stochastic Gradient Descent with Biased Regularization}

\author{ ~Giulia Denevi$^{1,2}$ ~~~~~ Carlo Ciliberto$^{3}$ ~~~~~~~ Riccardo Grazzi$^{1,4}$ ~~~~~ Massimiliano Pontil $^{1,4}$ \\ {\small \hspace*{-2.0em} ~~giulia.denevi@iit.it~~~carlo.ciliberto@imperial.ac.uk ~~~ riccardo.grazzi@iit.it ~~~~ massimiliano.pontil@iit.it}}

\begin{document}

\maketitle

\begin{abstract}
\noindent We study the problem of learning-to-learn: inferring a learning algorithm that works well on tasks sampled from an unknown distribution. 
As class of algorithms we consider Stochastic Gradient Descent on the true risk regularized by the square euclidean distance to a bias vector. We present an average excess risk bound for such a learning algorithm. This result quantifies the potential benefit of using a bias vector with respect to the unbiased case. We then address the problem of estimating the bias from a sequence of tasks.
We propose a meta-algorithm which incrementally updates the bias, as new tasks are observed. The low space and time complexity of this approach makes it appealing in practice. We provide guarantees on the learning ability of the meta-algorithm.  
A key feature of our results is that, when the number of tasks grows and their variance is relatively small, our learning-to-learn approach has a significant advantage over learning each task in isolation by Stochastic Gradient Descent without a bias term.
We report on numerical experiments which demonstrate the effectiveness of our approach. 
\end{abstract}


\section{Introduction}
\label{introduction_sec}

\footnotetext[1]{Computational Statistics and Machine Learning, Istituto Italiano di Tecnologia, 16163 Genoa, Italy}\footnotetext[2]{Department of Mathematics, University of Genoa, 16146 Genoa, Italy}\footnotetext[3]{Department of Electrical and Electronic Engineering, Imperial College of London, SW7 1AL, London, UK}\footnotetext[4]{Department of Computer Science, University College London, WC1E 6BT, London, UK}

The problem of learning-to-learn (LTL)~\cite{baxter2000model,thrun2012learning} is receiving increasing 
attention in recent years, due to its practical importance~\cite{pmlr-v70-finn17a,ravi2016optimization} 
and the theoretical challenge of statistically principled and efficient solutions 
\cite{alquier2016regret,balcan2015efficient,maurer2016benefit,pentina2014pac,denevi2018,denevi2018learning,gupta2017pac}.
The principal aim of LTL is to design a meta-learning algorithm to select a supervised learning algorithm that is well suited to learn tasks from a prescribed family. To highlight the difference between the meta-learning algorithm and the learning algorithm, throughout the paper we will refer to the latter as the {\em inner} or {\em within-task} algorithm. 

The meta-algorithm is trained from a sequence of datasets, associated with different learning tasks sampled from a meta-distribution (also called \emph{environment} in the literature). The performance of the selected inner algorithm is 
measured by the \emph{transfer risk} \cite{baxter2000model,maurer2005algorithmic}, that is, the average risk of the algorithm, trained on a random dataset from the same environment. 
A key insight is that, when the learning tasks share specific similarities, the LTL framework provides a means to leverage such similarities and select an inner algorithm of low transfer risk.

In this work, we consider environments of linear regression or binary classification tasks and we assume that the associated weight vectors are all close to a common vector. Because of the increasing interest in low computational complexity procedures, we focus on the family of within-task algorithms given by Stochastic Gradient Descent (SGD) working on the regularized true risk. Specifically, motivated by the above assumption on the environment, we consider as regularizer the square distance of the weight vector to a bias vector, playing the role of a common mean among the tasks. Knowledge of this common mean can substantially facilitate the inner algorithm and the main goal of this paper is to 
design a meta-algorithm to learn a good bias that is supported by both computational and statistical guarantees. 

{\bf Contributions.} The first contribution of this work is to show that, when the variance of the weight tasks' vectors sampled from the environment is small, SGD regularized with the ``right'' bias yields a model with smaller error than its unbiased counterpart when applied to a similar task. Indeed, the latter approach does not exploit the relatedness among the tasks, that is, it corresponds to learning the tasks in isolation -- also known as independent task learning  (ITL).  
The second and principal contribution of this work is to propose a meta-algorithm that estimates the 
bias term, so that the transfer risk of the corresponding SGD algorithm is as small as possible. 
Specifically, we consider the setting in which we receive in input a sequence of datasets and we propose an online meta-algorithm which efficiently updates the bias term used by the inner SGD algorithm. 
Our meta-algorithm consists in applying Stochastic Gradient Descent to a proxy of the transfer risk, given by the expected minimum value of the regularized empirical risk of a task. 
We provide a bound on the statistical performance of the biased SGD inner algorithm found by our procedure. 
It establishes that, when the number of observed tasks grows and the variance of the weight tasks' vectors is significantly smaller than their second moment, then, running the inner SGD algorithm with the estimated bias brings an improvement in comparison to learning the tasks in isolation with no bias. The bound is coherent with the state-of-the-art LTL analysis for other families of algorithms, but it applies for the first  time to a fully online meta-algorithm. Our results holds for Lipschitz loss functions both in the regression and binary classification setting.

Our proof techniques combines ideas from online learning, stochastic and convex optimization, 
with tools from LTL. A key insight in our approach is to exploit the inner SGD algorithm to compute an approximate subgradient of the surrogate objective, in a such way that the degree of approximation 
can be controlled, without affecting the overall performance or the computational
cost of the meta-algorithm.

{\bf Paper Organization.} We start from recalling in \shortautoref{preliminaries_sec} the basic 
concepts of LTL. In \shortautoref{inner_algorithm_sec} we cast the problem of
choosing a right bias term in SGD on the regularized objective in the LTL framework. 
Thanks to this formulation, in \shortautoref{advantage_bias_sec} we characterize the 
situations in which SGD with the right bias term is beneficial in comparison to SGD with no bias. In 
\shortautoref{estimating_bias_sec} we propose an online meta-algorithm to estimate the 
bias vector from a sequence of datasets and we analyze its statistical properties. In 
\shortautoref{exps_sec} we report on the empirical performance of the proposed 
approach while in \shortautoref{conclusions_sec} we discuss some future research directions.

\noindent {\bf Previous Work.} The LTL literature in the online setting \cite{alquier2016regret, denevi2018, denevi2018learning, pentina2016lifelong} has received limited attention and is less developed than standard LTL approaches, in which the data are processed in one batch as opposed to incrementally, see for instance ~\cite{baxter2000model,maurer2009transfer,maurer2013sparse,maurer2016benefit,pentina2014pac}. 
The idea of introducing a bias in the learning algorithm is not new, see e.g. \cite{denevi2018learning,kuzborskij2017fast,pentina2014pac} and the discussion in \shortautoref{inner_algorithm_sec}.
In this work, 
we consider the family of inner SGD algorithms with biased regularization and we develop a theoretically 
grounded meta-learning algorithm learning the bias. We are not aware of other works dealing with such a family in the LTL framework. Differently from others online methods \cite{alquier2016regret,denevi2018}, 
our approach does not need to keep previous training points in memory and it runs online both across and within the tasks. As a result, both the low space and time complexity are the strengths of our method. 


\section{Preliminaries}
\label{preliminaries_sec}

In this section, we recall the standard supervised (i.e. single-task)
learning setting and the learning-to-learn setting. 

We first introduce some notation used throughout.  We denote by $\Z = \X \times \Y$ the data space, 
where $\X \subset \R^d$ and $\Y \subseteq \R$ (regression) or $\Y = \{-1,+1\}$ (binary classification).
Throughout this work we consider linear supervised learning tasks $\task$, namely distributions over $\Z$,
parametrized by a weight vector $w \in \Real^d$.
We measure the performance by a loss function $\ell: \Y \times \Y \to 
\Real_+$ such that, for any $y \in \Y$, $\ell(\cdot,y)$ is convex and closed. 
Finally, for any positive $k \in \mathbb{N}$, we let $[k] = \{ 1, \dots, k \}$ and, we denote by 
$\langle \cdot, \cdot \rangle$ and $\| \cdot \|$ the standard inner product and euclidean norm. In the rest of this work, when specified, we make the following assumptions.

\begin{assumption}[Bounded Inputs] \label{bounded_inputs}
Let $\X \subseteq \mathcal{B}(0, \rx)$, where $\mathcal{B}(0, \rx) 
= \big \{ x \in \Real^d: \| x \| \le \rx \big \}$, for some radius $\rx \ge 0$. 
\end{assumption}
\begin{assumption}[Lipschitz Loss] \label{lipschitz_loss}
Let $\ell(\cdot, y)$ be $L$-Lipschitz for any $y \in \Y$.
\end{assumption}
For example, for any $y, \hat y \in \Y$, the absolute loss $\ell(\hat y, y) = \big | \hat y - y \big |$ 
and the hinge loss $\ell(\hat y, y) = \max \big \{ 0, 1 - y \hat y \big \}$ are both $1$-Lipschitz.
We now briefly recall the main notion of single-task learning.

\subsection{Single-Task Learning}
In standard linear supervised learning, 
the goal is to learn a linear functional relation $f_w : \X \to \Y$, $f_w(\cdot) = \langle \cdot, w \rangle$ 
between the input space $\X$ and the output space $\Y$. This target can be reformulated as the one of 
finding a weight vector $\wmu$ minimizing the \emph{expected risk} (or true risk)
\begin{equation} \label{ppo}
\cR_\mu(w) = \Exp_{(x,y)\sim\mu}~\ell \bigl( \langle x, w \rangle, y \bigr)
\end{equation}
over the \emph{entire} space $\Real^d$. The expected risk measures the prediction 
error that the weight vector $w$ incurs on average with respect to points sampled from the distribution $\task$. In practice, the task $\task$ is unknown and only partially 
observed by a corresponding dataset of $n$ i.i.d. points $\Zn = (z_i)_{i=1}^n \sim \task^n$, where, 
for every $i \in [n]$, $z_i = (x_i, y_i) \in \Z$. In the sequel, we often use the more compact notation 
$\Zn = (\Xn, \yn)$, where $\Xn \in \R^{n \times d}$ is the matrix containing the inputs vectors $x_i$
as rows and $\yn \in \R^n$ is the vector with entries given by the labels $y_i$. A {\em learning algorithm} 
is a function $A: \cup_{n \in \N} \Z^n \rightarrow \Real^d$ that, given such a {\em training}  
dataset $Z_n \in \Z^n$, returns a ``good'' estimator, that is, in our case, a weight vector $A(Z_n) \in \Real^d$, 
whose expected risk is small and tends to the minimum of Eq.~\eqref{ppo} as $n$ increases.

\subsection{Learning-to-Learn (LTL)}

In the LTL framework, we assume that each learning task $\task$ we observe is sampled from an 
\emph{environment} $\env$, that is a (meta-)distribution on the set of probability 
distributions on $\Z$. The goal is to select a learning algorithm (hence the name {\em learning-to-learn})
that is well suited to the environment. 

Specifically, we consider the following setting. We receive a 
stream of tasks $\task_1, \dots, \task_T$, which are independently sampled from the environment 
$\env$ and only partially observed by corresponding i.i.d. datasets $\datan^{(1)}, \dots, \datan^{(T)}, \dots$
each formed by $n$ datapoints. Starting from these datasets, we wish to learn an algorithm $A$, such that, when we apply it on 
a new dataset (composed by $n$ points) sampled from a new task $\task \sim \env$, the corresponding true risk is low. 
We reformulate this target into requiring that  algorithm $A$ trained with $n$ points\footnote{In order to simplify the presentation, we assume that all datasets are composed 
by the same number of points $n$. The general setting can be addressed introducing the slightly different
definition of the transfer risk $\E(A) = \Exp_{(n,\task) \sim \env}~\Exp_{\datan \sim \task^n} ~\cR_\task(A(\Zn))$. 
\label{footnote_different_n}} 
over the environment $\env$, has small {\em transfer risk}
\begin{equation} \label{eq:TRgen}
\begin{split}
\E_n(A) = \Exp_{\task \sim \env}~\Exp_{\datan \sim \task^n}~\cR_\task(A(\Zn)).
\end{split}
\end{equation}
The transfer risk measures the expected true risk that the inner algorithm $A$, trained on the dataset $\Zn$, incurs {\em on average with respect to the distribution of 
tasks $\mu$ sampled from $\rho$}. Therefore, the process of learning a learning algorithm is a meta-learning one, 
in that the inner learning algorithm is applied to tasks from the environment and then chosen from 
a sequence of training tasks (datasets) in attempt to minimize the transfer risk. 

As we will see in the following, in this work, we will consider a family of learning algorithms $A_h$
parametrized by a bias vector $h \in \Real^d$.


\section{SGD on the Biased Regularized Risk}
\label{inner_algorithm_sec}


In this section, we introduce the LTL framework for the family of within-task algorithms we analyze in this work: SGD on the biased regularized true risk.

The idea of introducing a bias in a specific family of learning algorithms is not new in the LTL literature, see e.g. \cite{denevi2018learning,kuzborskij2017fast,pentina2014pac} and references therein. A natural choice is given by regularized 
empirical risk minimization, in which we introduce a bias $h \in \mathbb{R}^d$ in the square norm regularizer -- which we simply 
refer to as ERM throughout -- namely 
\begin{equation} \label{RERM}
A^{\rm{ERM}}_h(\Zn) \equiv \whD(\Zn) = \argmin_{w \in \Real^d}~ \cR_{\Zn, h}(w),
\end{equation}
where, for any $w, h \in \Real^d$, $\la > 0$, we have defined the empirical 
error and its biased regularized version as
\begin{equation} \label{reg_emp}
\begin{split}
\cR_{\Zn}(w) & = \frac{1}{n} \sum_{k = 1}^n \ell_k \bigl( \langle x_k, w \rangle \bigr) \\
\cR_{\Zn,h}(w) & = \cR_{\Zn}(w) + \frac{\la}{2} \| w - h \|^2.
\end{split}
\end{equation}

Intuitively, if the weight vectors $\wmu$ of the tasks sampled from $\env$ are close to each other, then 
running ERM with $h = \wbar \equiv \Exp_{\task \sim \env} \wmu$ 
should have a smaller transfer risk than running ERM with, for instance, $h=0$. We make this statement precise in 
\shortautoref{advantage_bias_sec}. Recently, a number of works have considered how to learn a good bias $h$ 
in a LTL setting, see e.g. \cite{pentina2014pac, denevi2018learning}. However, one drawback of these works is 
that they assume the ERM solution to be known exactly, without leveraging the interplay between the optimization 
and the generalization error. Furthermore, in LTL settings, data naturally arise in an online manner, both 
{\em between} and {\em within} tasks. 
Hence, an ideal LTL approach should focus on inner algorithms processing one single data point at time.

Motivated by the above reasoning, in this work, we propose to analyze an online learning 
algorithm that is computationally and memory efficient while retaining (on average 
with respect to the sampling of the data) the {\em same statistical guarantees} of the more expensive 
ERM estimator. Specifically, for a training dataset $\Zn \sim \task^n$, a regularization 
parameter $\la > 0$ and a bias vector $h \in \H$, we consider the learning algorithm defined as 
\begin{equation}
A_h^{\rm{SGD}}(\Zn) \equiv \bwhD (\Zn),
\end{equation}
where, $\bwhD (\Zn)$ is the average of the first $n$ iterations of 
\shortautoref{Within-Task Algorithm Online_paper},
in which, for any $k \in [n]$, we have introduced the notation $\ell_k(\cdot) = \ell(\cdot, y_k)$.

\begin{algorithm}[t]
\begin{algorithmic}
\State ~
   \State {\bfseries Input} ~~ $\lambda > 0$ regularization parameter, $h$ bias, $\task$ task
   \vspace{.2cm}
   \State {\bfseries Initialization} ~~ $\whD^{(1)} = h$ 
  \vspace{.2cm}
   \State {\bfseries For} ~~ $k=1$ to $n$
    \vspace{.1cm}
    \State \qquad Receive ~~ $(x_k, y_k) \sim \task$
     \vspace{.1cm}
     \State \qquad Build ~~ $\ell_{k,h}(\cdot) = \ell_k( \langle x_k, \cdot \rangle ) 
    + \displaystyle \frac{\la}{2} \| \cdot - h \|^2$
   \vspace{.1cm}
   \State \qquad Define ~~ $\gamma_k = 1/ (k \la)$
   \vspace{.2cm}
   \State \qquad Compute $u'_{k} \in \partial \ell_k ( \langle x_k, \whD^{(k)} \rangle )$
   \vspace{.2cm}
    \State \qquad Define ~~ $s_k = x_k u'_{k} + \la (\whD^{(k)} - h) \in \partial \ell_{k,h} ( \whD^{(k)})$
   \vspace{.0001cm}
   \State \qquad Update ~~ $\displaystyle \whD^{(k+1)} = \whD^{(k)} - \gamma_k s_k$ 
   \vspace{.2cm}
 \State {{\bfseries Return} ~~ $( \whD^{(k)})_{k=1}^{n+1}$,~ $\displaystyle \bwhD = \frac{1}{n} \sum_{i =1}^n \whD^{(i)}$} 
\State ~
\end{algorithmic}
\caption{Within-Task Algorithm: SGD on the Biased Regularized True Risk}\label{Within-Task Algorithm Online_paper}
\end{algorithm}

\shortautoref{Within-Task Algorithm Online_paper} coincides with 
online subgradient algorithm applied to the strongly convex function $\cR_{\Zn, h}$.
Moreover, thanks to the assumption that $\Zn \sim \task^n$, \shortautoref{Within-Task Algorithm Online_paper}
is equivalent to SGD applied to the regularized true risk
\begin{equation} \label{reg_true_risk}
\cR_{\task,h}(w) = \cR_\task(w) + \frac{\la}{2} \| w - h \|^2.
\end{equation}

Relying on standard online-to-batch argument, see e.g. \cite{cesa2004generalization,hazan2016introduction} and references therein, it is 
easy to link the true error of such an algorithm with the minimum of the regularized 
empirical risk, that is, $\cR_{\Zn, h}(\whD(\Zn))$. This fact is reported in the proposition 
below and it will be often used in our subsequent statistical 
analysis. We give a proof in \shortautorefsubapp{statistical_preliminaries_app} 
for completeness.

\begin{restatable}{proposition}{estimation} \label{estimation_error}
Let \shortautoref{bounded_inputs} and \shortautoref{lipschitz_loss} hold and let  
$\bwhD$ be the output of \shortautoref{Within-Task Algorithm Online_paper}. 
Then, we have that
\begin{equation}
\begin{split}
\Exp_{\Zn \sim \task^n}~\big[ \cR_\task& \bigl( \bwhD(\Zn) \bigr) - \cR_{\Zn,h}(\whD(\Zn)) \big] \le c_{n, \la} \\ \\
& c_{n, \la} = \frac{2 \rx^2 L^2 \bigl( {\rm{log}}(n) + 1 \bigr)}{\la n}.
\end{split}
\end{equation}
\end{restatable}
We remark that at this level of the analysis, one could also avoid the logarithmic factor in the above 
bound, see e.g. \cite{shamir2013stochastic,rakhlin2012making,lacoste2012simpler}. However, in order 
to not complicate our presentation and proofs, we avoid this refinement of the analysis.

In the next section 
we study the impact on the bias vector on the statistical performance of the inner algorithm.
Specifically, we investigate under which circumstances there is an advantage in perturbing the regularization 
in the objective used by the algorithm with an appropriate ideal bias term $h$, as opposed to fix $h = 0$. 
Throughout the paper, we refer to the choice $h = 0$ as independent task learning (ITL), although strictly 
speaking, when $h$ is fixed in advanced, then, SGD is applied on each task independently regardless of the value of $h$.
Then, in \shortautoref{estimating_bias_sec} we address the question of estimating this appropriate bias from the data.


\section{The Advantage of the Right Bias Term}
\label{advantage_bias_sec}

In this section, we study the statistical performance of the model $\bwhD$ returned by \autoref{Within-Task Algorithm Online_paper}, on average with 
respect to the tasks sampled from the environment $\env$, for different choices 
of the bias vector $h$. 
To present our observations, we require, for any $\task \sim \env$, that the corresponding true risk admits minimizers 
and we denote by $\wmu$ the minimum norm minimizer\footnote{This choice is made in order to simplify our presentation. 
However, our analysis holds for different choices of a minimizer $\wmu$, which may potentially improve our bounds.}. 
With these ingredients, we introduce the {\rm oracle}
\begin{equation*}
\eeinf = \EE_{\task \sim \env}~\cR_\task(\wmu),
\end{equation*}
representing the averaged minimum error over the environment of tasks, and, for a candidate bias $h$, 
we give a bound on the quantity $\E(\bwhD) - \eeinf$. This gap coincides with the 
averaged excess risk of algorithm \shortautoref{Within-Task Algorithm Online_paper} 
with bias $h$ over the environment of tasks, that is
\begin{equation*}
\een(\bwhD) - \eeinf = \Exp_{\task \sim \env}~\Exp_{\Zn \sim \task^n}~
\big[ \cR_\task \bigl( \bwhD(\Zn) \bigr) - \cR_\task(w_\task) \big].
\end{equation*}
Hence, this quantity is an indicator of the performance of the bias $h$ with 
respect to our environment. In the rest of this section, we study the above gap 
for a bias $h$ which is fixed and does not depend on the data. 
Before doing this, we introduce the notation
\begin{equation} \label{var_h}
{\rm Var}_h^2 = \frac{1}{2}~\EE_{\task \sim \env}~\| w_\task - h \big \|^2
\end{equation}
which is used throughout this work and we observe that
\begin{equation} \label{mean_w}
\wbar \equiv \Exp_{\task \sim \env} \wmu = \argmin_{h \in \Real^d}~{\rm Var}_h^2.
\end{equation}

\begin{restatable}[Excess Transfer Risk Bound for a Fixed Bias $h$]{theoremshortref}{excessriskITL} \label{excess_transfer_risk_fixed_h}
Let \shortautoref{bounded_inputs} and \shortautoref{lipschitz_loss} hold and let  
$\bwhD$ be the output of \shortautoref{Within-Task Algorithm Online_paper} with
regularization parameter
\begin{equation} \label{best_lambda}
\la = \frac{\rx L}{{\rm Var}_h}~\sqrt{\frac{2 \bigl( {\rm{log}}(n) + 1 \bigr)}{n}}. 
\end{equation}
Then, the following bound holds
\begin{equation}
\een(\bwhD) - \eeinf \le {\rm Var}_h~2 \rx L~\sqrt{\frac{2 \bigl( {\rm{log}}(n) + 1 \bigr)}{n}}.
\end{equation}
\end{restatable}

\begin{proof} 
For $\task \sim \env$, consider the following decomposition 
\begin{equation} \label{dec}
\Exp_{\Zn \sim \task^n}~\cR_\task(\bwhD(\Zn)) - \cR_\task(\wmu) \le \text{A} + \text{B},
\end{equation}
where, A and B are respectively defined by
\begin{equation} \label{est}
\begin{split}
\text{A} & = \Exp_{\Zn \sim \task^n}~\big[\cR_{\task}(\bwhD(\Zn)) - \cR_{\Zn,h}(\whD(\Zn)) \big] \\
\text{B} & = \Exp_{\Zn \sim \task^n}~\big[\cR_{\Zn,h}(\whD(\Zn)) - \cR_{\task}(w_{\task}) \big].
\end{split}
\end{equation}
In order to bound the term A, we use \shortautoref{estimation_error}.
Regarding the term B, we exploit the definition of the ERM algorithm
and the fact that, since $\wmu$ does not depend on $\Zn$, then $\cR_{\task,h}(\wmu) 
= \Exp_{\Zn \sim \task^n}~\cR_{\Zn,h}(\wmu)$. Consequently, we can upper bound the
term B as 
\begin{equation} \label{approx}
\begin {split}
& \Exp_{\Zn \sim \task^n}\big[\cR_{\Zn,h}(\whD(\Zn)) {-} \cR_{\task,h}(w_{\task}) \big] 
+ \frac{\la}{2} \big \| w_\task {-} h \big \|^2\\
& = \Exp_{\Zn \sim \task^n}\big[\cR_{\Zn,h}(\whD(\Zn)) {-} \cR_{\Zn,h}(w_{\task}) \big] 
+ \frac{\la}{2} \big \| w_\task {-} h \big \|^2 \\
& \le \frac{\la}{2} \big \| w_\task {-} h \big \|^2.
\end{split}
\end{equation}
The desired statement follows by combining the above bounds on the two terms, taking 
the average with respect to $\task \sim \env$ and optimizing over $\la$.
\end{proof}

\shortautoref{excess_transfer_risk_fixed_h} shows that the strength of the regularization that one should use in 
the within-task algorithm \shortautoref{Within-Task Algorithm Online_paper}, is inversely proportional to both the 
variance of the bias $h$ and the number of points in the datasets. This is exactly in line with the LTL aim: when solving 
each task is difficult, knowing a priori a good bias can bring a substantial benefit over learning with no bias.
To further investigate this point, in the following corollary, we specialize \shortautoref{excess_transfer_risk_fixed_h} to 
two particular choices of the bias $h$ which are particularly meaningful for our analysis. The first choice we make is $h = 0$, 
which coincides, as remarked earlier, with learning each task independently, while the second choice considers an ideal bias, 
namely, assuming that the transfer risk admits minimizer, we set $h=\hn \in \argmin_{h \in \Real^d}~\een(\bwhD)$.

\begin{corollary}[Excess Transfer Risk Bound for ITL and the Oracle] \label{excess_transfer_risk_ITL_Oracle}
Let \shortautoref{bounded_inputs} and \shortautoref{lipschitz_loss} hold. 
\begin{enumerate}
\item {\bf Independent Task Learning.}  Let $\bar{w}_0$ be the output of 
\shortautoref{Within-Task Algorithm Online_paper} with bias $h = 0$ 
and regularization parameter as in Eq. \eqref{best_lambda} with $h = 0$.
Then, the following bound holds
\begin{equation*}
\een(\bar{w}_0) - \eeinf \le {\rm Var}_{0}~2 \rx L~\sqrt{\frac{2 \bigl( {\rm{log}}(n) + 1 \bigr)}{n}}.
\end{equation*}
\item {\bf The Oracle.} Let $\bar{w}_{\hn}$ be the output of \shortautoref{Within-Task Algorithm Online_paper} 
with bias $h = \hn$ and regularization parameter as in Eq. \eqref{best_lambda} with $h = \wbar$. 
Then, the following bound holds
\begin{equation*}
\een(\bar{w}_{\hn}) - \eeinf \le {\rm Var}_{\wbar}~2 \rx L~\sqrt{\frac{2 \bigl( {\rm{log}}(n) + 1 \bigr)}{n}}.
\end{equation*}
\end{enumerate}
\end{corollary}

\begin{proof} The proof of the first statement directly follows from the application of \shortautoref{excess_transfer_risk_fixed_h} with $h = 0$. The second statement
is a direct consequence of the definition of $\hn$ implying $\een(\bar{w}_{\hn}) - 
\eeinf \le \een(\bar{w}_{\wbar}) - \eeinf$ and the application of 
\shortautoref{excess_transfer_risk_fixed_h} with $h = \wbar$ on the second term.
\end{proof}

From the previous bounds we can observe that, using the bias $h = \hn$ in the regularizer 
brings a substantial benefit with respect to the unbiased case when the number of points $n$ 
in each dataset in not very large 
(hence learning each task is quite difficult) and the variance of the weight tasks' vectors sampled
from the environment is much smaller than their second moment, i.e. when
\begin{equation*}
{\rm Var}_{\wbar}^2 = \frac{1}{2}~\EE_{\task \sim \env} ~ \| \wmu - \wbar \|^2 \ll  
\frac{1}{2}~\EE_{\task \sim \env} ~ \| \wmu \|^2 = {\rm Var}_0^2.
\end{equation*}
Driven by this observation, when the environment of tasks satisfies the above 
characteristics, we would like to take advance of this tasks' similarity. But, 
since in practice we are not able to explicitly compute $\hn$, in the following 
section we propose an efficient online LTL approach to estimate the bias 
directly from the observed data sequence of tasks.


\section{Estimating the Bias}
\label{estimating_bias_sec}
In this section, we study the problem of designing an estimator for the bias vector that is 
computed \emph{incrementally} from a set of observed $T$ tasks.

\subsection{The Meta-Objective}

Since direct optimization of the transfer risk is not feasible, a standard strategy used in LTL 
consists in introducing a proxy objective that is easier to handle, see e.g. \cite{maurer2005algorithmic,
maurer2009transfer,maurer2013sparse,maurer2016benefit, denevi2018, denevi2018learning}. 
In this paper, motivated by \shortautoref{estimation_error}, according to which
\begin{equation*}
\begin{split}
\Exp_{\Zn \sim \task^n}~\big[ \cR_\task \bigl( \bwhD(\Zn) \bigr) \big] \le
\Exp_{\Zn \sim \task^n}~\big[\cR_{\Zn,h}(\whD(\Zn)) \big] 
+ \frac{2 \rx^2 L^2 \bigl( {\rm{log}}(n) + 1 \bigr)}{\la n},
\end{split}
\end{equation*}
we substitute in the definition of the transfer risk the true risk of the algorithm
$\cR_\task \bigl( \bwhD(\Zn) \bigr)$ with the minimum of the regularized 
empirical risk
\begin{equation} \label{primal_min_val_paper}
\begin{split}
\LL_{\Zn}(h) = \min_{w \in \Real^d} ~ \cR_{\Zn,h}(w) = \cR_{\Zn,h}(\whD(\Zn)).
\end{split}
\end{equation}
This leads us to the following proxy for the transfer risk
\begin{equation} \label{proxy}
\hat \E_n(h) = \Exp_{\task \sim \env}~\Exp_{\Zn \sim \task^n}~ \LL_{\Zn}(h).
\end{equation} 

Some remarks about this choice are in order. First, convexity is usually a rare property in LTL. In our case, 
as described in the following proposition, the definition of the function $\LL_{\Zn}$ as the partial minimum 
of a jointly convex function, ensures convexity and other nice properties, such as differentiability and a 
closed expression of its gradient.

\begin{restatable}[Properties of $\LL_{\Zn}$]{proposition}{properties} \label{properties}
The function $\LL_{\Zn}$ in Eq. \eqref{primal_min_val_paper} is convex and $\la$-smooth 
over $\H$. Moreover, for any $h \in \H$, its gradient is given by the formula
\begin{equation} \label{exact_meta_gradient}
\nabla \LL_{\Zn}(h) = - \la \bigl( \whD(\Zn) - h \bigr), 
\end{equation}
where $\whD(\Zn)$ is the ERM algorithm in Eq. \eqref{RERM}. Finally, when 
\shortautoref{bounded_inputs} and \shortautoref{lipschitz_loss} hold, $\LL_{\Zn}$ 
is $L \rx$-Lipschitz.
\end{restatable}

The above statement is a known result in the optimization community,
see e.g. \cite[Prop. $12.29$]{bauschke2011convex} and
\shortautoref{properties_meta_obj_app} for more details.
In order to minimize the proxy objective in Eq. \eqref{proxy}, one standard choice done in stochastic 
optimization, and also adopted in this work, is to use first-order methods, requiring the computation of 
an unbiased estimate of the gradient of the stochastic objective. 
In our case, according to the above proposition, this step would require 
computing the minimizer of the regularized empirical problem in Eq. \eqref{primal_min_val_paper} exactly.
A key observation of our work is to show below that we can easily design a ``satisfactory'' approximation 
(see the last paragraph in \shortautoref{estimating_bias_sec}) of its gradient, just substituting the 
minimizer $\whD(\Zn)$ in the expression of the gradient in Eq. \eqref{exact_meta_gradient} with the 
last iterate  $\whD^{(n+1)}(\Zn)$ of \shortautoref{Within-Task Algorithm Online_paper}. 
An important aspect to stress here is the fact that this strategy does not require any additional 
computational effort. Formally, this reasoning is linked to the concept of $\epsilon$-subgradient of
a function. We recall that, for a given convex, proper and closed function $f$
and for a given point $\hat h \in \dom(f)$ in its domain, $u$ is an $\epsilon$-subgradient 
of $f$ at $\hat h$, if, for any $h$, $f(h) \ge f( \hat h) + \langle u, h-\hat h \rangle - \epsilon$. \\

\begin{restatable}[An $\epsilon$-Subgradient for $\LL_{\Zn}$]{proposition}{epsilonsubgradient} 
\label{epsilon-subgradient}
Let $\whD^{(n+1)}(\Zn)$ be the last iterate of \shortautoref{Within-Task Algorithm Online_paper}.
Then, under \shortautoref{bounded_inputs} and \shortautoref{lipschitz_loss}, the vector
\begin{equation} \label{approx_meta_gradient}
\hat \nabla \LL_{\Zn}(h) = - \la \bigl( \whD^{(n+1)}(\Zn) - h \bigr)
\end{equation}
is an $\epsilon$-subgradient of $\LL_{\Zn}$ at point $h$, where $\epsilon$ is such that
\begin{equation} \label{pp}
\Exp_{\Zn \sim \task^n}~\big[ \epsilon \big] \le \frac{2 \rx^2 L^2 \bigl( {\rm{log}}(n) + 1 \bigr)}{\la n}.
\end{equation}
Moreover, introducing $\Delta_{\Zn}(h) = \nabla \LL_{\Zn}(h) - \hat \nabla \LL_{\Zn}(h)$, 
\begin{equation} \label{ppp}
\Exp_{\Zn \sim \task^n}~\big \| \Delta_{\Zn}(h) \big \|^2 
\le \frac{4 \rx^2 L^2 \bigl( {\rm{log}}(n) + 1 \bigr)}{n}.
\end{equation}
\end{restatable}

The above result is a key tool in our analysis. The proof requires some preliminaries on the $\epsilon$-subdifferential of a 
function (see \shortautoref{epsilon_subgradients_app}) and introducing the dual formulation of both the within-task learning 
problem and \shortautoref{Within-Task Algorithm Online_paper} (see \shortautoref{primal_dual_formulation_app} 
and \shortautoref{primal_dual_formulation_SGD_app}, respectively). With these two ingredients, the proof of the statement is deduced in \shortautorefsubapp{proof_prop_epsilon_subgradient_app} by the application of a more general result reported in \shortautoref{meta_approx_subgradient_app}, describing how an $\epsilon$-minimizer of the dual of the within-task learning problem can be exploited in order to build an $\epsilon$-subgradient of the meta-objective function $\LL_{\Zn}$. We stress that 
this result could be applied to more general class of algorithms, going beyond  
\shortautoref{Within-Task Algorithm Online_paper} considered here.


\subsection{The Meta-Algorithm to Estimate the Bias $h$} 

In order to estimate the bias $h$ from the data, we apply SGD to the
stochastic function $\hat \E_n$ introduced in Eq. \eqref{proxy}.
More precisely, in our setting, the sampling of a ``meta-point'' corresponds to
the incremental sampling of a dataset from the environment\footnote{More precisely we 
first sample a distribution $\mu$ from $\rho$ and then a dataset $Z_n \sim \mu^n$.}.  
We refer to \shortautoref{OGDA2_paper} for more details. In particular,
we propose to take the estimator $\bar h_T$ obtained by averaging the
iterations returned by \shortautoref{OGDA2_paper}. An important feature to stress 
here is the fact that the meta-algorithm uses $\epsilon$-subgradients of the function 
$\LL_{\Zn}$ which are computed as described above. Specifically, for any 
$t \in [T]$, we define
\begin{equation} \label{approx_meta_gradient_t}
\hat \nabla \LL_{\Zn^{(t)}}(h^{(t)}) = - \la \bigl( w_{h^{(t)}}^{(n+1)}(\Zn^{(t)}) - h^{(t)} \bigr),
\end{equation}
where $w_{h^{(t)}}^{(n+1)}$ is the last iterate of \shortautoref{Within-Task Algorithm Online_paper}
applied with the current bias $h^{(t)}$ and the dataset $\Zn^{(t)}$. To simplify the presentation, 
throughout this work, we use the short-hand notation 
\begin{equation*}
\LL_ t (\cdot) = \LL_{\Zn^{(t)}}(\cdot), ~
\nabla^{(t)} = \nabla \LL_t(h^{(t)}), ~
\hat \nabla^{(t)} = \hat \nabla \LL_t(h^{(t)}).
\end{equation*}
Some technical observations follows. First, we stress that \shortautoref{OGDA2_paper} processes 
one single instance at the time, without the need to store previously encountered data points, neither 
across the tasks nor within them. Second, the implementation of \shortautoref{OGDA2_paper}
does not require computing the meta-objective $\LL_{\Zn}$, which would increase the computational
effort of the entire scheme. The rest of this section is devoted to the statistical analysis 
of \shortautoref{OGDA2_paper}.

\begin{algorithm}[t]
\caption{Meta-Algorithm, SGD on $\hat \E$ with $\epsilon$-Subgradients}\label{OGDA2_paper}
\begin{algorithmic}
\State ~
   \State {\bfseries Input} $\gamma > 0$ step size, $\la > 0$ inner regularization parameter, $\env$ meta-distribution
   \vspace{.2cm}
   \State {\bfseries Initialization} ~~ $h^{(1)} = 0 \in \H$
  \vspace{.2cm}
   \State {\bfseries For} ~~ $t=1$ to $T$
   \vspace{.1cm}
   \State \qquad ~~ Receive ~~ $\task_t \sim \env$, $\Zn^{(t)} \sim \task_t^n$
   \vspace{.2cm}
   \State \qquad ~~ Run the inner algorithm \shortautoref{Within-Task Algorithm Online_paper} and approximate
    \vspace{.05cm}
   \State \qquad ~~ the gradient 
   $\hat \nabla^{(t)} \approx \nabla^{(t)}$ by Eq. \eqref{approx_meta_gradient_t}
   \vspace{.2cm}
   \State \qquad ~~ Update ~~ $h^{(t+1)} = h^{(t)} - \gamma \hat \nabla^{(t)}$
   \vspace{.2cm}
 \State {{\bfseries Return} ~~ $( h^{(t)})_{t=1}^{T+1}$ and $\displaystyle \bar h_T = \frac{1}{T} \sum_{t=1}^T h^{(t)}$} 
\State ~
\end{algorithmic}
\end{algorithm}


\subsection{Statistical Analysis of the Meta-Algorithm}

\noindent In the following theorem we study the statistical performance of the bias
$\bar h_T$ returned by \shortautoref{OGDA2_paper}.
More precisely we bound the excess transfer risk of the inner SGD algorithm ran 
with this biased term learned by the meta-algorithm.

\begin{restatable}[Excess Transfer Risk Bound for the Bias $\bar h_T$ Estimated by \shortautoref{OGDA2_paper}]
{theoremshortref}{excessriskLTL} \label{excess_transfer_risk_LTL}
Let \shortautoref{bounded_inputs} and \shortautoref{lipschitz_loss} hold and
let $\bar h_T$ be the output of \shortautoref{OGDA2_paper} with step size
\begin{equation}
\gamma = \displaystyle \frac{\sqrt{2} \| \wbar \|}{L \rx}~\sqrt{
\Bigl(T \Bigl( 1 + \frac{4 \bigl( {\rm{log}}(n) + 1 \bigr)}{n} \Bigr) \Bigr)^{-1}}.
\end{equation}
Let $\bar w_{\bar h_T}$ be the output of \shortautoref{Within-Task Algorithm Online_paper} 
with bias $h = \bar h_T$ and regularization parameter 
\begin{equation}
\la = \frac{2 \rx L}{{\rm Var}_{\wbar}}~\sqrt{\frac{{\rm{log}}(n) + 1}{n}}.
\end{equation}
Then, the following bound holds
\begin{equation*}
\begin{split}
\Exp \left[ \een(\bar w_{\bar h_T}) \right] - \eeinf \le {\rm Var}_{\wbar} ~ 4 \rx L  
~\sqrt{\frac{{\rm{log}}(n) + 1}{n}}
+ \| \wbar \| ~ L \rx~\sqrt{2 \Bigl( 1 + \frac{4 \bigl( {\rm{log}}(n) + 1 \bigr)}{n} \Bigr) \frac{1}{T}},
\end{split}
\end{equation*}
where the expectation above is with respect to the sampling 
of the datasets $\Zn^{(1)}, \dots, \Zn^{(T)}$ from the environment $\env$.
\end{restatable}

\begin{proof} 
We consider the following decomposition
\begin{equation}
\Exp \left[ \een(\bar w_{\bar h_T}) \right]- \eeinf \le \text{A} + \text{B} + \text{C},
\end{equation}
where we have defined the terms
\begin{equation}
\begin{split}
\text{A} & = \een(\bar w_{\bar h_T}) - \hat \E_n(\bar h_T) \\
\text{B} & = \Exp ~ \hat \E_n(\bar h_T) -  \hat \E_n(\wbar) \\
\text{C} & = \hat \E_n(\wbar) - \eeinf.
\end{split}
\end{equation}
Now, in order to bound the term A, noting that
\begin{equation*}
\text{A} = \Exp_{\task \sim \env}~\Exp_{\Zn \sim \task^n}~\big[ \cR_\task \bigl( \bar{w}_{\bar h_T}(\Zn) \bigr)
- \cR_{\Zn, \bar h_T}(w_{\bar h_T}(\Zn)) \big],
\end{equation*}
we use \shortautoref{estimation_error} with $h = \bar h_T$ and, then,
we take the average on $\task \sim \env$. As regards the term C, we apply 
the inequality given in Eq. \eqref{approx} with $h = \wbar$ and we again
average with respect to $\task \sim \env$. Finally, the term B is the convergence 
rate of \shortautoref{OGDA2_paper} and its study requires analyzing the error that we 
introduce in the meta-gradients by \shortautoref{epsilon-subgradient}. The bound
we use for this term is the one described in \shortautoref{regret_OGDA_proj} (see \shortautorefsubapp{estimated_h_generalization_bound_proof_app}) with 
$\hat h = \wbar$. The result now follows by combining the bounds on the three terms and 
optimizing over $\la$.
\end{proof}

We remark that the bound in \shortautoref{excess_transfer_risk_LTL} is stated with respect 
to the mean $\wbar$ of the tasks' vector only for simplicity, and the same result holds for a generic bias vector $h \in \H$. Specializing this rate to ITL ($h = 0)$ recovers the rate in 
\shortautoref{excess_transfer_risk_ITL_Oracle} for ITL (up to a contant $2$). 
Consequently, even when the tasks are not “close to each other” (i.e. their 
variance ${\rm Var_{\bar w}}^2$ is high), our approach is not prone to negative-transfer, 
since, in the worst case, it recovers the ITL performance. 
Moreover, the above bound is coherent with the state-of-the-art LTL bounds given in 
other papers studying other variants of Ivanov or Tikhonov regularized empirical risk 
minimization algorithms, see e.g. \cite{maurer2005algorithmic, maurer2009transfer,maurer2013sparse,maurer2016benefit}. 
Specifically, in our case, the bound has the form
\begin{equation}
\mathcal{O}\Bigl( \frac{{\rm Var}_{\wbar}}{\sqrt{n}} \Bigr)+ \mathcal{O}\Bigl( \frac{1}{\sqrt{T}} \Bigr),
\end{equation} 
where ${\rm Var}_{\wbar}$ reflects the advantage in exploiting the relatedness among 
the tasks sampled from the environment $\env$. More precisely, in \shortautoref{advantage_bias_sec} 
we noticed that, if the variance of the weight vectors of the tasks sampled from our 
environment is significantly smaller than their second moment, running \shortautoref{Within-Task Algorithm Online_paper} 
with the ideal bias $h = \hn$ on a future task brings a significant improvement in comparison to the
unbiased case. One natural question arising at this point of the presentation is whether, 
under the same conditions on the environment, the same improvement is obtained by running 
\shortautoref{Within-Task Algorithm Online_paper} 
with the bias vector $h = {\bar h}_T$ returned by our online meta-algorithm in 
\shortautoref{OGDA2_paper}. Looking at the bound in 
\shortautoref{excess_transfer_risk_LTL}, we can say that,
when the number of training tasks $T$ used to estimate the 
bias $\bar h_T$ is sufficiently large, the above question has
a positive answer and our LTL approach is effective.

In order to have also a more precise benchmark for the biased setting
considered in this work, 
in \shortautorefsubapp{ERM_analysis} we have repeated the 
statistical study described in the paper also for the more expensive ERM
algorithm described in Eq. \eqref{RERM}. In this case, we assume to have 
an oracle providing us with this exact estimator, ignoring  
any computational costs. As before, we have performed the analysis 
both for a fixed bias and the one estimated from the data which is returned by 
running \shortautoref{OGDA2_paper}. We remark that, thanks to the
assumption on the oracle, in this case, \shortautoref{OGDA2_paper} is
assumed to run with exact meta-gradients. Looking at the results 
reported in \shortautorefsubapp{ERM_analysis}, we immediately see that, 
up to constants and logarithmic factors, the LTL bounds we have stated in the 
paper for the low-complexity SGD family are equivalent to the ones we have 
reported in \shortautorefsubapp{ERM_analysis} for the more 
expensive ERM family. 

All the above facts justify the informal statement given before 
\shortautoref{epsilon-subgradient} according to which the trick used to compute 
the approximation of the meta-gradient by using the last iterate  of the inner 
algorithm, not only, does not require additional effort, but it is also accurate enough 
from the statistical view point, matching a state-of-the-art bound 
for more expensive within-task algorithms based on ERM.

We conclude by observing that, exploiting the explicit form of the
error on the meta-gradients, it is possible to extend  the analysis presented 
in \shortautoref{excess_transfer_risk_LTL} above to the adversarial case,
where no assumption on the data generation process is made. The result
in our statistical setting can be derived from this more general adversarial 
setting by the application of two online-to-batch conversions, one within-task
and one outer-task.


\section{Experiments}
\label{exps_sec}

In this section, we test the effectiveness of the LTL approach proposed in this paper on synthetic and real data 
\footnote{The code used for the following experiments is available at \emph{https://github.com/prolearner/onlineLTL}}.
In all experiments, the regularization parameter $\la$ and the step-size $\gamma$ were tuned by validation, see \shortautoref{validation_app} for more details. 

{\bf Synthetic Data.} We considered two different settings, regression with the absolute loss and binary classification with 
the hinge loss. In both cases, we generated an environment of tasks in which SGD with the right bias is expected to bring 
a substantial benefit in comparison to the unbiased case. Motivated by our observations in \shortautoref{advantage_bias_sec},
we generated linear tasks with weight vectors characterized by a variance which is significantly smaller than their second moment.
Specifically, for each task $\task$, we created a weight vector $\wmu$ from a Gaussian distribution with mean $\wbar$ given by 
the vector in $\mathbb{R}^d$ with all components equal to $4$ and standard deviation ${\rm Var}_{\wbar} = 1$. Each task corresponds to a dataset $(x_i,y_i)_{i=1}^n$, $x_i\in\mathbb{R}^d$ with 
$n = 10$ and $d = 30$. In the regression case, the 
inputs were uniformly sampled on the unit sphere and the labels were generated as $y = \langle x,\wmu \rangle + \epsilon$, with $\epsilon$ sampled from a zero-mean Gaussian distribution, with standard deviation chosen to have signal-to-noise ratio equal to $10$ for each task. In the classification case, the inputs were uniformly sampled on the unit sphere, excluding those points with margin 
$|\langle x, \wmu \rangle|$ smaller than $0.5$ and the binary labels were generated as 
a logistic model, $\mathbb{P}(y = 1) = \bigl(1 + 10 \, {\rm{exp}}(- \langle x, \wmu \rangle) \bigr)^{-1}$. In Fig.~\ref{fig:synth-data}
 we report the performance of \shortautoref{Within-Task Algorithm Online_paper} with different choices of the bias: $h = \bar h_T$ (our LTL estimator resulting from  \shortautoref{OGDA2_paper}), $h = 0$ (ITL) and $h = \wbar$, 
a reasonable approximation of the oracle minimizing the transfer risk.  
The plots confirm our theoretical findings: estimating the bias with our LTL approach leads to a substantial benefits with respect to the unbiased case, as the number of the observed 
training tasks increases.

\begin{figure}[t] 
\begin{center}
    \includegraphics[width=0.48\textwidth]{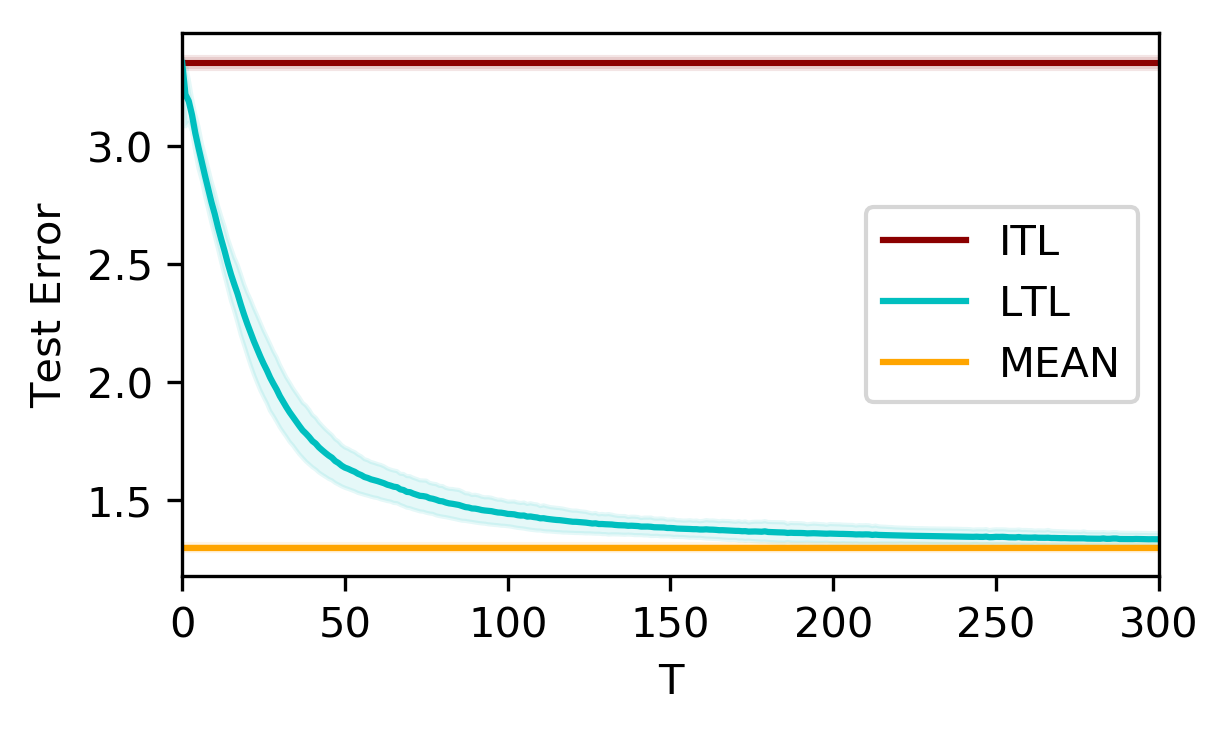} \quad 
    \includegraphics[width=0.48\textwidth]{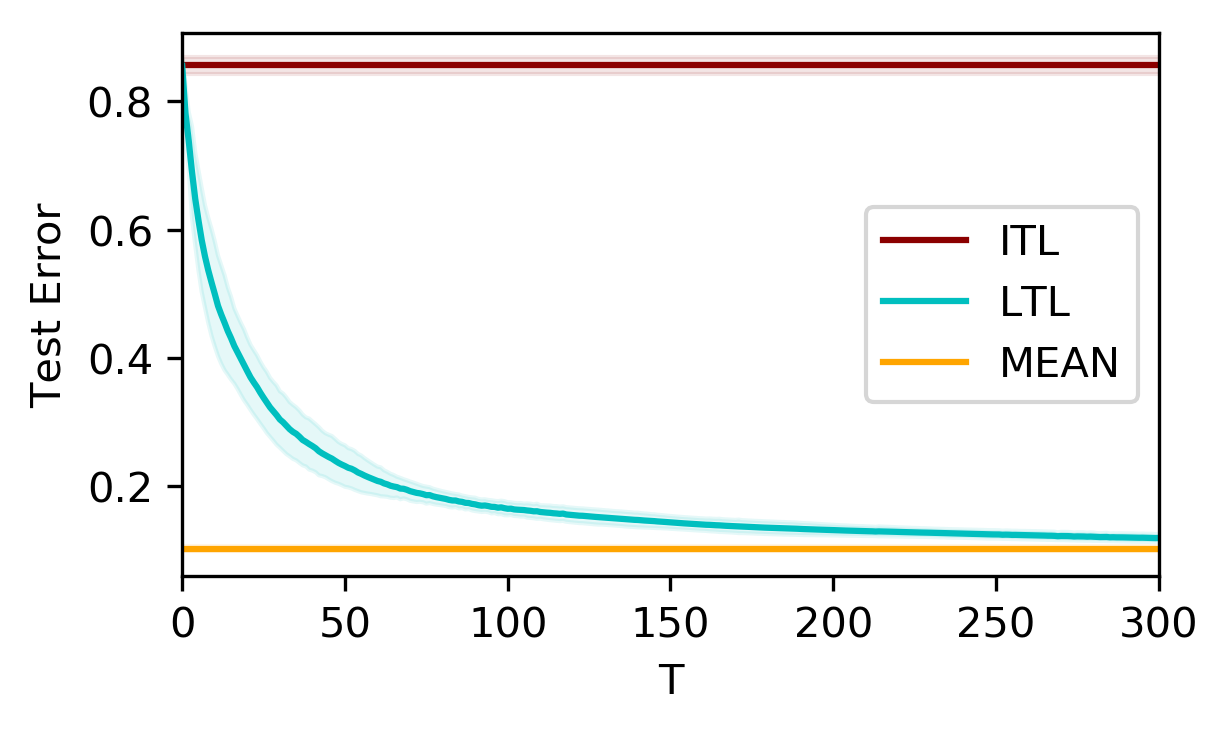}
    \caption{{\bf Synthetic Data.} Test performance of different bias with respect to an increasing number 
of tasks. (Top) Regression with absolute loss. (Bottom) Classification with hinge loss. The results are averaged 
over $10$ independent runs (datasets generations).\label{fig:synth-data}}
\end{center}
\end{figure}

\begin{figure}[t] 
\begin{center}
    \includegraphics[width=0.48\textwidth]{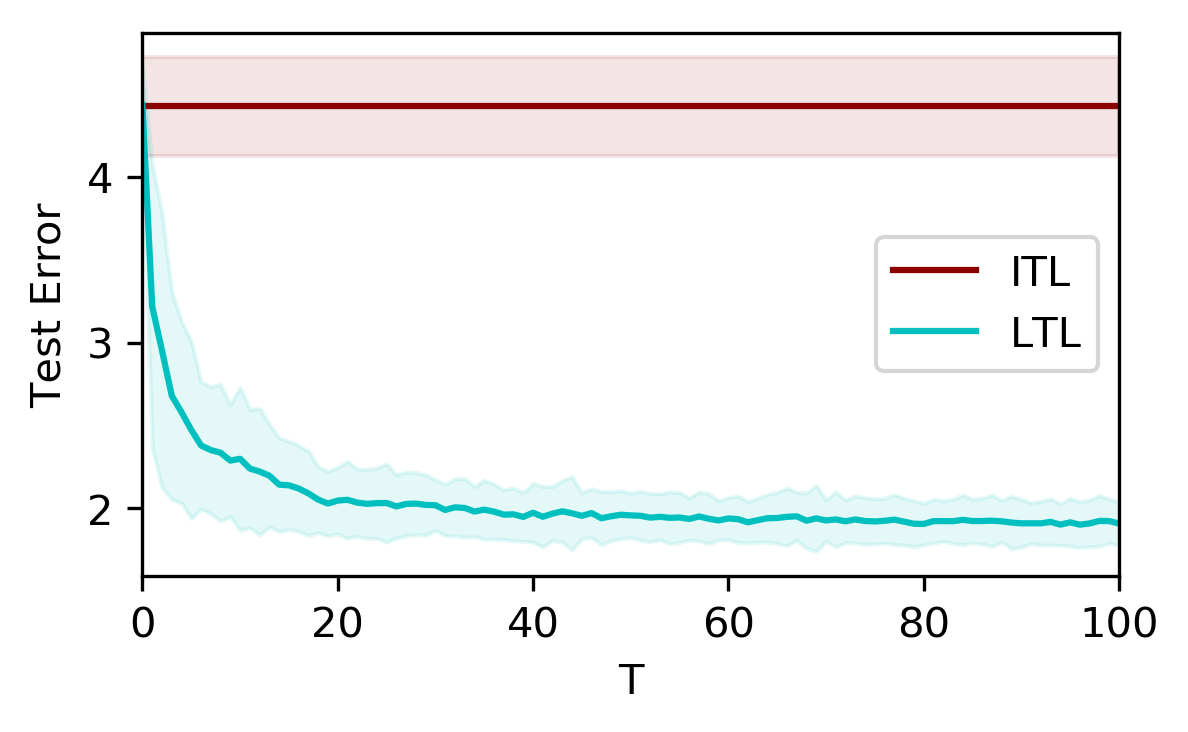} \quad 
    \includegraphics[width=0.48\textwidth]{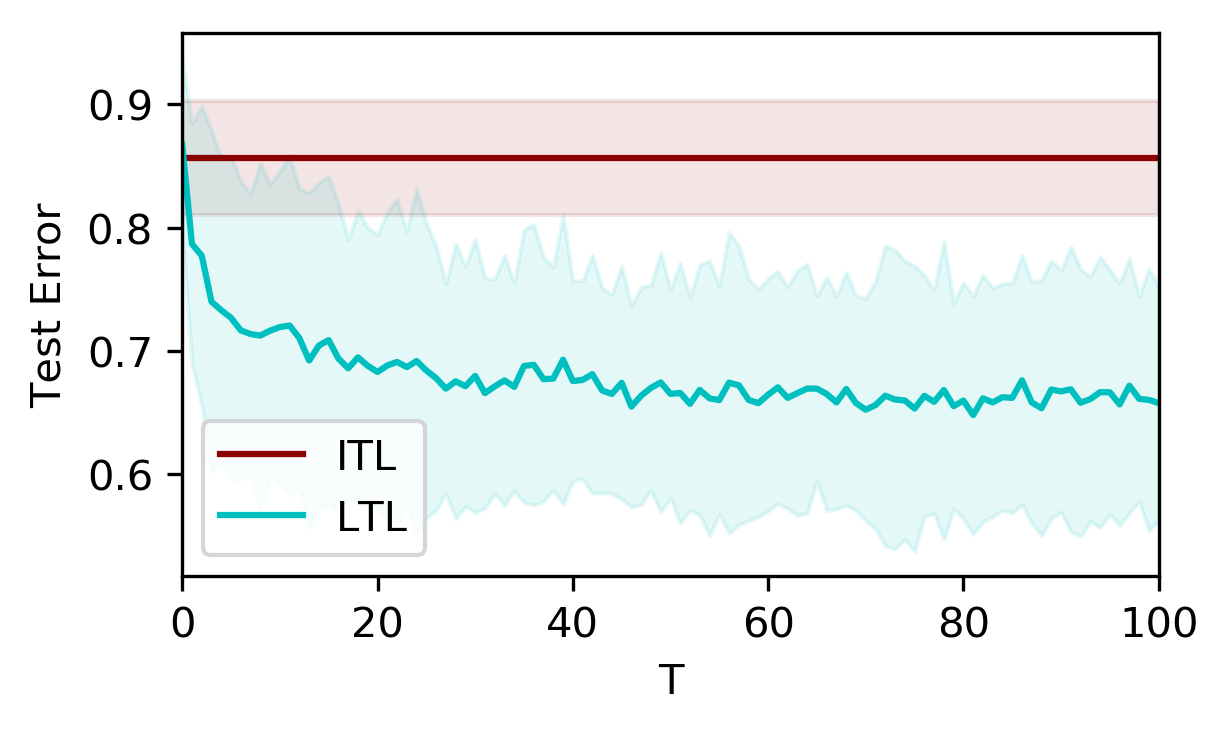}
    \caption{{\bf Real Data.} Test performance of different bias with respect to an increasing number of tasks. (Top) Lenk Dataset Regression. (Bottom) Lenk Dataset Classification. The results are averaged over $30$ independent runs (datasets generations).\label{fig:real-data}}
\end{center}
\end{figure}

{\bf Real Data.} We run experiments on the computer survey data from \cite{lenk1996hierarchical}, 
in which 180 people (tasks) rated the likelihood of purchasing one of 20 different
personal computers ($n=8$). The input represents 13 different computer characteristics (price, CPU, RAM, etc.) while the output
is an integer rating from $0$ to $10$. Similarly to the synthetic data experiments, we consider a regression setting with the absolute loss and a classification setting. In the latter case each task is to predict whether the rating is above $5$.
We compare the LTL bias with ITL. The results are reported in Fig.~\ref{fig:real-data}. 
The figures above are in line with the results obtained on synthetic experiments, indicating that the bias LTL framework proposed in this work is effective for this dataset. Moreover, the results for regression are also in line with what observed in the multitask setting with variance regularization \cite{Andrew}. The classification setting has not been used before and has been created ad-hoc for our purpose. In this case we have an increased variance probably due to the datasets being highly unbalanced.
In order to investigate the impact of passing through the data only once in the different steps in our method, we conducted additional experiments. The results, presented in  \shortautorefsubapp{additional_exps}, indicate that the single pass strategy 
is competitive with respect to the more expensive ERM.


\section{Conclusion and Future Work}
\label{conclusions_sec}
We have studied the performance of Stochastic Gradient Descent on the true risk regularized 
by the square euclidean distance to a bias vector, over a class of tasks. 
Drawing upon a learning-to-learn framework, we have shown that, when the variance of the
tasks is relatively small, the introduction of an appropriate bias vector could be beneficial in 
comparison to the standard unbiased version, corresponding to learning the tasks independently. 
Then, we have proposed an efficient online meta-learning algorithm to estimate this bias and we 
have theoretically shown that the bias returned by our method can bring a comparable benefit.
In the future, 
it would be interesting to investigate other kinds of relatedness among the tasks and to extend our analysis to other classes of loss functions, as well as to a Hilbert space setting. 
Finally, another valuable research direction is to derive fully dependent bounds, in which the hyperparameters are self-tuned during the learning process, see e.g. \cite{zhuang2019surrogate}.



\bibliography{giulia-2}
\bibliographystyle{abbrv}

\newpage

\begin{appendices}

\section*{Appendix}

The appendix is organized as follows. In \shortautorefsubapp{epsilon_subgradients_app} 
we report some basic facts regarding the $\epsilon$-subdifferential of a function which are used
in the subsequent analysis. In \shortautorefsubapp{primal_dual_formulation_app} we give the 
primal-dual formulation of the biased regularized empirical risk minimization problem for each single task
and, in \shortautorefsubapp{properties_meta_obj_app}, we recall some well-known properties
of our meta-objective function. In \shortautorefsubapp{meta_approx_subgradient_app}, 
we show how an $\epsilon$-minimizer of the dual problem can be exploited in order to build 
an $\epsilon$-subgradient of our meta-objective function. As described in 
\shortautorefsubapp{primal_dual_formulation_SGD_app}, interpreting our within-task algorithm 
as a coordinate descent algorithm on the dual problem, we can adapt this result to our setting 
and prove, in this way, \shortautoref{epsilon-subgradient}. 
In \shortautorefsubapp{statistical_preliminaries_app}, we report the proof of 
\shortautoref{estimation_error} and, in \shortautorefsubapp{estimated_h_generalization_bound_proof_app},
we give the convergence rate of \shortautoref{OGDA2_paper} which is used in the paper, 
during the proof of \shortautoref{excess_transfer_risk_LTL}.
In \shortautorefsubapp{ERM_analysis}, we repeat the statistical study described in the 
paper also for the family of ERM algorithms introduced in Eq. \eqref{RERM} and, in \shortautorefsubapp{validation_app},
we describe how to perform the validation procedure in our LTL setting. Finally, in \shortautorefsubapp{additional_exps} we report additional experiments comparing our
method to ERM variants.


\section{Basic Facts on $\epsilon$-Subgradients}
\label{epsilon_subgradients_app}

In this section, we report some basic concepts about the $\epsilon$-subdifferential
which are then used in the subsequent analysis. This material is based 
on \cite[Chap. XI]{jean2010convex}. Throughout this section we consider a convex 
closed and proper function $f: \Real^d \to \Real \cup \{+\infty\}$ with domain 
$\dom(f)$ and we always let $\epsilon \ge 0$. 

\begin{definition}[$\epsilon$-Subgradient, {\cite[Chap. XI, Def. $1.1.1$]{jean2010convex}}] \label{def_epsilon_sub}
Given $\hat h \in \dom(f)$, the vector $u \in \Real^d$ is called $\epsilon$-subgradient of $f$ 
at $\hat h$ when the following property holds for any $h \in \Real^d$
\begin{equation} \label{sub_gradient_def}
f(h) \ge f( \hat h) + \langle u, h-\hat h \rangle - \epsilon.
\end{equation}
The set of all $\epsilon$-subgradients of f at $\hat h$ is the $\epsilon$-subdifferential of $f$ at 
$\hat h$, denoted by $\partial_\epsilon f(\hat h)$.
\end{definition}


The standard subifferential $\partial f(\hat h)$ is retrieved with $\epsilon = 0$.
The following lemma, which is a direct consequence of \shortautoref{def_epsilon_sub},
points out the link between $\partial_\epsilon f$ and an $\epsilon$-minimizer of $f$.

\begin{lemma}[See {\cite[Chap. XI, Thm. $1.1.5$]{jean2010convex}}] \label{epsilon_minimum}
The following two properties are equivalent.
\begin{equation}
0 \in \partial_\epsilon f(\hat h) \quad \iff \quad f(\hat h) \le f(h) + \epsilon \quad \text{for any $h \in \Real^d$}.
\end{equation}
\end{lemma}

The subsequent lemma describes the behavior of the $\epsilon$-subdifferential with respect to the duality.

\begin{lemma}[See {\cite[Chap. XI, Prop. $1.2.1$]{jean2010convex}}] \label{duality_epsilon_subgradient}
Let $f^*: \Real^d \to \Real \cup \{ + \infty \}$ be the Fenchel conjugate of $f$, namely, 
$f^*(\cdot) = \sup_{h \in \Real^d} \langle \cdot, h \rangle - f(h)$. Then, given $\hat h \in \dom(f)$, 
the vector $u \in \Real^d$ is an $\epsilon$-subgradient of $f$ at $\hat h$ iff
\begin{equation}
f^*(u) + f(\hat h) - \langle u, \hat h \rangle \le \epsilon.
\end{equation}
As a result,
\begin{equation}
u \in \partial_\epsilon f(\hat h) \quad \iff \quad \hat h \in \partial_\epsilon f^*(u).
\end{equation}
\end{lemma}

We now describe some properties of the $\epsilon$-subdifferential which are
used in the following analysis.

\begin{lemma}[ See {\cite[Chap. XI, Thm. $3.1.1$]{jean2010convex}}] \label{epsilon_subgradient_sum}
Let $f_1$ and $f_2$ be two convex closed and proper functions. Then, given $\hat h \in \dom(f_1+f_2) =
\dom(f_1) \cap \dom(f_2)$, we have that
\begin{equation}
\bigcup_{0 \le \epsilon_1 + \epsilon_2 \le \epsilon} \partial_{\epsilon_1} f_1( \hat h) + \partial_{\epsilon_2} f_2( \hat h)
\subset \partial_\epsilon \bigl( f_1 + f_2 \bigr) ( \hat h ).
\end{equation}
Moreover, denoting by ${\rm{ri}}(A)$ the relative interior of a set $A$, when 
${\rm{ri}}\bigl( \dom(f_1) \bigr) \cap {\rm{ri}}\bigl( \dom(f_2) \bigr) \ne \emptyset$, equality holds.
\end{lemma}


\begin{lemma}[ See {\cite[Chap. XI, Prop. $1.3.1$]{jean2010convex}}] \label{epsilon_subgradient_chain_rule_partialscalar}
Let $a \ne 0$ be a scalar. Then, for a given $\hat h \in \dom(f \circ a)$, we have that
\begin{equation}
\partial_\epsilon \bigl( f \circ a \bigr) \bigl( \hat h \bigr) = a ~ \partial_\epsilon f \bigl( a \hat h \bigr).
\end{equation}
\end{lemma}

\begin{lemma} \label{epsilon_subgradient_chain_rule_partial}
Let $X \in \Real^{n \times d}$ be a matrix. Then, for a given $\hat h \in \Real^d$ such that $X \hat h \in \dom(f)$,
we have that
\begin{equation}
X \trans \partial_\epsilon f \bigl( X \hat h \bigr) \subset \partial_\epsilon \bigl( f \circ X \bigr) (\hat h).
\end{equation}
\end{lemma}

\begin{proof}
Let be $u \in X \trans \partial_\epsilon f \bigl( X \hat h \bigr)$. Then, by definition, there exist $v \in 
\partial_\epsilon f \bigl( X \hat h \bigr)$ such that $u = X \trans v$. Consequenlty, for any $h \in \Real^d$,
we can write
\begin{equation}
\big \langle u, h - \hat h \big \rangle = \big \langle X \trans v, h - \hat h \big \rangle =
\big \langle v, X h - X \hat h \big \rangle \le f \bigl( X h \bigr) - f \bigl( X \hat h \bigr) + \epsilon
= \bigl( f \circ X \bigr) (h) - \bigl( f \circ X \bigr) (\hat h) + \epsilon,
\end{equation}
where, in the inequality we have used the fact that $v \in \partial_\epsilon f \bigl( X \hat h \bigr)$. 
This gives the desired statement.
\end{proof}

The next two results characterize the $\epsilon$-subdifferential of two  
functions, which are useful in our subsequent analysis. In the following we denote 
by $\psd$ the set of the $d \times d$ symmetric positive semi-definite matrices.

\begin{example}[Quadratic Functions, {\cite[Chap. XI, Ex. $1.2.2$ ]{jean2010convex}}] \label{epsilon_subgradient_quadratic}
For a given matrix $Q \in \psd$ and a given vector $b \in \Real^d$, consider the function 
\begin{equation}
f : h \in \Real^d \mapsto \frac{1}{2} \big \langle Q h, h\big \rangle + \langle b, h \rangle.
\end{equation}
Then, given $\hat h \in \dom(f) = \Real^d$, we can express the $\epsilon$-subdifferential of $f$ at $\hat h$
with respect to the gradient $\nabla f(\hat h) = Q \hat h + b$ as follows
\begin{equation}
\partial_\epsilon f(\hat h) = \Big \{ \nabla f(\hat h) + Q s : \frac{1}{2} \big \langle Q s, s \big \rangle \le \epsilon \Big \}.
\end{equation}
\end{example}

\begin{example}[Moreau Envelope {\cite[Chap. XI, Ex. $3.4.4$]{jean2010convex}}] \label{epsilon_subgradient_Moreau_envelope}
For $\la > 0$ and a fixed vector $h \in \Real^d$, consider the Moreau envelope of $f$ at the point 
$h$ with parameter $\la$, given by
\begin{equation}
\LL(h) = \min_{w \in \Real^d} f(w) + \frac{\la}{2} \big \| w - h \big \|^2.
\end{equation}
Denote by $\whD$ the unique minimizer of the above function, namely, the 
vector characterized by the optimality conditions
\begin{equation}
0 \in \partial f(\whD) + \la \bigl( \whD - h \bigr).
\end{equation}
Then, for any $\la > 0$ and $h \in \Real^d$, we have that
\begin{equation} \label{epsilon_subdifferential_Moreau_envelope}
\partial_\epsilon \LL (h) = \bigcup_{0 \le \alpha \le \epsilon}
\partial_{\epsilon-\alpha} f( \whD ) \cap \mathcal{B} \Bigl( - \la \bigl( \whD - h \bigr),
\sqrt{2 \la \alpha} \Bigr),
\end{equation}
where, for any center $c \in \Real^d$ and any radius $r \ge 0$, we recall the notation 
\begin{equation}
\mathcal{B} (c,r) = \big \{ u \in \Real^d: \| u - c \| \le r \big \}.
\end{equation}
For $\epsilon = 0$ we retrieve the well-known result according to which $\LL$ 
is differentiable, with $\la$-Lipschitz gradient given by
\begin{equation}
\nabla \LL( h) = - \la \bigl( \whD - h \bigr).
\end{equation}
Finally, from Eq. \eqref{epsilon_subdifferential_Moreau_envelope}, we can 
deduce that, if $u \in \partial_\epsilon \LL (h)$, then 
\begin{equation} \label{discrepancy_gradients}
\big \| \nabla \LL( h) - u \big \| \le \sqrt{2 \la \epsilon}.
\end{equation}
\end{example}


\section{Primal-Dual Formulation of the Within-Task Problem}
\label{primal_dual_formulation_app}

In this section, we give the primal-dual formulation of the biased regularized empirical risk 
minimization problem outlined in Eq, \eqref{RERM} for each single task. Specifically, rewriting 
for any $w \in \Real^d$ and $u \in \Real^n$, the empirical risk
\begin{equation}
\cR_{\Zn}(w) = \bigl( g \circ \Xn \bigr)(w) \quad \quad \quad 
g(u) = \frac{1}{n} \sum_{k = 1}^n \ell_k(u_k),
\end{equation}
for any $h \in \H$, we can express our meta-objective function in Eq. \eqref{primal_min_val_paper} as
\begin{equation} \label{primal_min_val}
\begin{split}
\LL_{\Zn}(h) = \min_{w \in \Real^d} ~ \bigl(g \circ \Xn \bigr)(w) ~ + ~ \frac{\la}{2} ~ \| w-h \|^2.
\end{split}
\end{equation}
We remark that, in the optimization community, this function coincides with the Moreau envelope of the 
empirical error at the point $h$, see also \shortautorefex{epsilon_subgradient_Moreau_envelope}. 
In this section, in order to simplify the presentation, we omit the dependence on the dataset $\Zn$ in 
the notation. The unique minimizer of the above function 
\begin{equation} \label{primal_sol}
\whD = \argmin_{w \in \Real^d} ~ \bigl(g \circ \Xn \bigr)(w) + ~ \frac{\la}{2} ~ \| w - h \|^2
\end{equation}
is known as the proximity operator of the empirical error at the point $h$ and it coincides
with the ERM algorithm introduced in Eq. \eqref{RERM} in the paper.
We interpret the vector $\whD$ in Eq. \eqref{RERM}--\eqref{primal_sol} as the solution of the primal problem
\begin{equation} \label{primal_problem}
\whD = \argmin_ {w \in \Real^d} \prhD(w) \quad \quad  \quad \quad 
\prhD(w) = \bigl(g \circ \Xn \bigr)(w) ~ + ~ \frac{\la}{2} ~ \| w - h \|^2.
\end{equation}
The next proposition is a standard result stating that, in this setting, strong
duality holds and the optimality conditions, also known as Karush--Kuhn--Tucker 
(KKT) conditions provide a unique way to determine the primal variables from the dual ones.

\begin{proposition}[Strong Duality, {\cite[Thm. $4.4.2$]{borwein2005techniques}, \cite[Prop. $15.18$]{bauschke2011convex}}] \label{strong_duality}
Consider the primal problem in Eq. \eqref{primal_problem}.
Then, its dual problem admits a solution 
\begin{equation} \label{dual_problem}
\uhD \in \argmin_ {u \in \Real^n} \duhD(u) \quad \quad \quad 
\duhD(u) = g^*(u) + \frac{1}{2 \la} \big \| \Xn \trans u \big \|^2 - 
\big \langle \Xn h, u \big \rangle,
\end{equation}
where, thanks to the separability of $g$, for any $u \in \Real^n$, we have that
\begin{equation}
g^*(u) = \frac{1}{n} \sum_{k =1}^n \ell_k^*(n u_k).
\end{equation}
Moreover, strong duality holds, namely, 
\begin{equation}
\LL(h) = \prhD(\whD) = \min_{w \in \Real^d} \prhD(w) = - \min_{u \in \Real^n} \duhD(u) = - \duhD(\uhD) 
\end{equation}
and the optimality (KKT) conditions read as follows
\begin{equation}
\begin{split}
\whD = - \frac{1}{\la} \Xn \trans \uhD + h & \iff \la ( \whD - h ) = - \Xn \trans \uh \\
\uhD \in \partial g(\Xn \whD) & \iff \Xn \whD \in \partial g^*(\uhD).
\end{split}
\end{equation}
\end{proposition}


\section{Properties of the Meta-Objective}
\label{properties_meta_obj_app}

In this section we recall some properties of the meta-objective function $\LL_{\Zn}$ already 
outlined in the text in \shortautoref{properties}.

\properties*

\begin{proof} 
The first part of the statement is a well-known fact, see \cite[Prop. $12.29$]{bauschke2011convex}
and also \shortautorefex{epsilon_subgradient_Moreau_envelope}.
In order to prove the second part of the statement, we exploit \shortautoref{bounded_inputs}
and \shortautorefex{lipschitz_loss} and we proceed as follows. According to the change of variables 
$v = w-h$, exploiting the fact that, for any two convex functions $f_1$ and $f_2$, we have
\begin{equation}
\Big | \min_{v \in \Real^d} f_1(v) - \min_{v \in \Real^d} f_2(v) \Big | \le 
\sup_{v \in \Real^d} \big | f_1(v) - f_2(v) \big |,
\end{equation}
for any $h_1, h_2 \in \H$, we can write the following
\begin{equation}
\begin{split}
\Big | & \LL_{\Zn}(h_1) - \LL_{\Zn}(h_2) \Big | \\
& = \Big | \min_{w \in \Real^d} ~ \Bigl( \frac{1}{n} \sum_{k=1}^n \ell_k\bigl( \langle x_k, w\rangle \bigr)
+ ~ \frac{\la}{2} ~ \| w - h_1 \|^2 \Bigr) - \min_{w \in \Real^d} ~ \Bigl( \frac{1}{n} \sum_{k=1}^n \ell_k\bigl( \langle x_k, w \rangle \bigr) + ~ \frac{\la}{2} ~ \| w - h_2 \|^2 \Bigr) \Big | \\
& = \Big | \min_{v \in \Real^d} ~ \Bigl( \frac{1}{n} \sum_{k=1}^n \ell_k\bigl( \langle x_k, v + h_1 \rangle \bigr)
+ ~ \frac{\la}{2} ~ \| v \|^2 \Bigr) - \min_{v \in \Real^d} ~ \Bigl( \frac{1}{n} \sum_{k=1}^n \ell_k\bigl( \langle x_k, v + h_2 
\rangle \bigr)
+ ~ \frac{\la}{2} ~ \| v \|^2 \Bigr) \Big | \\
& \le \sup_{v \in \Real^d} \Big | \frac{1}{n} \sum_{k=1}^n \ell_k\bigl( \langle x_k, v + h_1 \rangle \bigr)
+ ~ \frac{\la}{2} ~ \| v \|^2 - \frac{1}{n} \sum_{k=1}^n \ell_k\bigl( \langle x_k, v + h_2 \rangle \bigr)
- ~ \frac{\la}{2} ~ \| v \|^2 \Big | \\
& = \sup_{v \in \Real^d} \Big | \frac{1}{n} \sum_{k=1}^n \Bigl( \ell_k\bigl( \langle x_k, v + h_1 \rangle \bigr)
- \ell_k\bigl( \langle x_k, v + h_2 \rangle \bigr) \Bigr) \Big | \\
& \le \sup_{v \in \Real^d} \frac{1}{n} \sum_{k=1}^n \Big | \ell_k\bigl( \langle x_k, v + h_1 \rangle \bigr)
- \ell_k\bigl( \langle x_k, v + h_2 \rangle \bigr) \Big | \\
& \le \frac{L}{n} \sup_{v \in \Real^d} \sum_{k=1}^n \Big | \langle x_k, v +  h_1 \rangle - \langle x_k, v + h_2 \rangle \Big | \\
& = \frac{L}{n} \sum_{k=1}^n \Big | \langle x_k, h_1 - h_2 \rangle \Big | \\
& \le \frac{L}{n} \sum_{k=1}^n  \| x_k \| \| h_1 - h_2 \| \\
& \le L \rx \| h_1 - h_2 \|,
\end{split}
\end{equation}
where, in the third inequality we have used \shortautoref{lipschitz_loss}, in the fourth inequality we
have applied Cauchy-Schwartz inequality and in the last step we have used \shortautoref{bounded_inputs}. 
Consequently, we can state that $\LL_{\Zn}$ is $L \rx$-Lipschitz.
\end{proof}

To conclude this section, in the next proposition, we recall the closed form of the 
conjugate of the function $\LL_{\Zn}$.

\begin{lemma}[Fenchel Conjugate of $\LL_{\Zn}$] \label{conjugate_Moreau}
For any $\alpha \in \Real^d$, the Fenchel conjugate function of $\LL_{\Zn}$ is
\begin{equation}
\LL_{\Zn}^*(\alpha) = \bigl( g \circ \Xn \bigr)^*(\alpha) + \frac{1}{2 \la} \big \| \alpha \big \|^2.
\end{equation}
\end{lemma}

\begin{proof}
We recall that the infimal convolution of two proper closed convex functions $f_1$ 
and $f_2$ is defined as $\big( f_1~\square~f_2 \big) (\cdot) = \inf_{w} f_1(w) + f_2(\cdot - w)$
and its Fenchel conjugate is give by $\big( f_1~\square~f_2 \big)^* = f_1^* + f_2^*$, 
see \cite[Chap. XII]{bauschke2011convex}.
Hence, the statement follows from observing that, for any $h \in \H$ and any $\alpha \in \Real^d$,
$\displaystyle \LL_{\Zn}(h) = \bigl( g \circ \Xn \bigr)~\square~\frac{\la}{2} \| \cdot \|^2 (h)$ 
and $\displaystyle \Bigl( \frac{\la}{2} \| \cdot \|^2 \Bigr)^*(\alpha) = \frac{1}{2 \la} \| \alpha \|^2$.
\end{proof}


\section{From the Dual an $\epsilon$-Subgradient for the Meta-Objective}
\label{meta_approx_subgradient_app}

In this section, we show how to exploit an $\epsilon$-minimizer $\hat u_h$ of the dual problem in Eq. 
\eqref{dual_problem} in order to get an $\epsilon$-subgradient of the function $\LL_{\Zn}$ in Eq. 
\eqref{primal_min_val_paper}--\eqref{primal_min_val} at the point $h$. This is described in the following 
proposition, which will play a fundamental role in our analysis.

\begin{proposition}[$\epsilon$-Subgradient for the Meta-Objective $\LL_{\Zn}$] \label{approx_subgradient_prop}
In the setting described above, for a fixed value $h \in \H$ and a fixed parameter $\la > 0$, consider an 
$\epsilon$-minimizer $\hat u_h \in \Real^n$ of the dual objective $\duhD$ in Eq. \eqref{dual_problem}, for some 
value $\epsilon \ge 0$. Then, the vector $\Xn \trans \hat u_h \in \Real^d$ is an $\epsilon$-subgradient of $\LL_{\Zn}$ 
at the point $h$.
\end{proposition}

\begin{proof}
By \shortautoref{epsilon_minimum}, the assumption that $\hat u_h$ is an $\epsilon$-minimizer of $\duhD$ is 
equivalent to the condition $0 \in \partial_\epsilon \duhD(\hat u_h)$. Now recall that, 
for any $u \in \Real^n$, the expression of the dual objective is given by
\begin{equation}
\duhD(u) = g^*(u) + \frac{1}{2 \la} \big \| \Xn \trans u \big \|^2 - \big \langle \Xn h, u \big \rangle.
\end{equation}
Consequently, thanks to \shortautoref{epsilon_subgradient_sum}, for any 
$u \in \dom(\duhD) = \dom(g^*)$, we have that
\begin{equation}
\partial_\epsilon \duhD(u) = \bigcup_{0 \le \epsilon_1 + \epsilon_2 \le \epsilon}
\partial_{\epsilon_1} g^*(u) + \partial_{\epsilon_2} \Big \{ \frac{1}{2 \la} \big \| \Xn \trans \cdot \big \|^2 
- \big \langle \Xn h, \cdot \big \rangle \Big \} (u).
\end{equation}
Thanks to \shortautorefex{epsilon_subgradient_quadratic}, for any $u \in \Real^n$, 
we can write
\begin{equation}
\partial_{\epsilon_2} \Big \{ \frac{1}{2 \la} \big \| \Xn \trans \cdot \big \|^2 
- \big \langle \Xn h, \cdot \big \rangle \Big \}(u) = 
\Big \{ \Xn \Bigl( \frac{\Xn \trans u}{\la} - h + \frac{\Xn \trans s}{\la} \Bigl)~:~\frac{1}{2} 
\Big \langle \frac{\Xn \Xn \trans s}{\la}, s \Big \rangle \le \epsilon_2 \Big \}.
\end{equation}
Hence, we know that $0 \in \partial_\epsilon \duhD(\hat u_h)$ iff
\begin{equation}
\exists~\epsilon_1, \epsilon_2, s \in \Real^n~:~0 \le \epsilon_1 + \epsilon_2 \le \epsilon,~
\frac{1}{2} \Big \langle \frac{\Xn \Xn \trans s}{\la}, s \Big \rangle \le \epsilon_2
\end{equation}
such that the following relations hold true
\begin{equation} \label{gigia}
\begin{split}
0 \in \partial_{\epsilon_1} g^*(\hat u_h) + \Xn \Bigl( \frac{\Xn \trans \hat u_h}{\la} - h + \frac{\Xn \trans s}{\la} \Bigl) 
& \iff \Xn \Bigl( h - \frac{\Xn \trans (\hat u_h + s)}{\la} \Bigl)  \in \partial_{\epsilon_1} g^*(\hat u_h) \\
\text{\shortautoref{duality_epsilon_subgradient}} & \iff \hat u_h \in \partial_{\epsilon_1} g \Bigl( \Xn \Bigl( h - 
\frac{\Xn \trans (\hat u_h + s)}{\la} \Bigl) \Bigr) \\
& \implies \Xn \trans \hat u_h \in 
\Xn \trans \partial_{\epsilon_1} g \Bigl( \Xn \Bigl( h - \frac{\Xn \trans (\hat u_h + s)}{\la} \Bigl) \Bigr) \\
\text{\shortautoref{epsilon_subgradient_chain_rule_partial}} & \implies \Xn \trans \hat u_h \in \partial_{\epsilon_1} 
\bigl( g \circ \Xn \bigr) \Bigl( h - \frac{\Xn \trans (\hat u_h + s)}{\la} \Bigr) \\
\text{\shortautoref{duality_epsilon_subgradient}} & \iff  h - \frac{\Xn \trans (\hat u_h + s)}{\la} \in 
\partial_{\epsilon_1} \bigl( g \circ \Xn \bigr)^* \bigl( \Xn \trans \hat u_h \bigr) \\
& \iff h \in \partial_{\epsilon_1} \bigl( g \circ \Xn \bigr)^* 
\bigl( \Xn \trans \hat u_h \bigr) + \frac{\Xn \trans (\hat u_h + s)}{\la}.
\end{split}
\end{equation}
Now, thanks to \shortautoref{conjugate_Moreau}, we have that, for any $\alpha \in \Real^d$, 
the Fenchel conjugate function of $\LL_{\Zn}$ is
\begin{equation}
\LL_{\Zn}^*(\alpha) = \bigl( g \circ \Xn \bigr)^*(\alpha) + \frac{1}{2 \la} \big \| \alpha \big \|^2.
\end{equation}
Hence, thanks to \shortautoref{epsilon_subgradient_sum}, for any 
$\alpha \in \dom(\LL_{\Zn}^*) = \dom \bigl( \bigl( g \circ \Xn \bigr)^* \bigr) \supset \Xn \trans \dom(g^*)$, 
we have that
\begin{equation}
\partial_\epsilon \LL_{\Zn}^*(\alpha) = \bigcup_{0 \le \epsilon_1 + \epsilon_2 \le \epsilon}
\partial_{\epsilon_1} \bigl( g \circ \Xn \bigr)^*(\alpha) + \partial_{\epsilon_2} \Big \{ 
\frac{1}{2 \la} \big \| \cdot \big \|^2 \Big \} (\alpha).
\end{equation}
Moreover, thanks to \shortautorefex{epsilon_subgradient_quadratic}, we observe that
\begin{equation}
\partial_{\epsilon_2} \Big \{ \frac{1}{2 \la} \big \| \cdot \big \|^2 \Big \}(\alpha) = 
\Big \{ \frac{\alpha + \tilde s}{\la} ~:~\frac{1}{2 \la} \big \| \tilde s \big \|^2 \le \epsilon_2 \Big \}.
\end{equation}
Therefore, making the identification $\tilde s = \Xn \trans s$, the last relation in Eq. \eqref{gigia} tells us
\begin{equation}
\begin{split}
0 \in \partial_\epsilon \duhD(\hat u_h) & \implies h \in \partial_{\epsilon_1} \bigl( g \circ \Xn \bigr)^* 
\bigl( \Xn \trans \hat u_h \bigr) + \frac{\Xn \trans (\hat u_h + s)}{\la} \implies h \in \partial_\epsilon \LL_{\Zn}^*
(\Xn \trans \hat u_h) \\ & \iff \Xn \trans \hat u_h \in \partial_\epsilon \LL_{\Zn}(h),
\end{split}
\end{equation}
where, in the last equivalence, we have used again \shortautoref{duality_epsilon_subgradient}.
This proves the desired statement.
\end{proof}


\section{SGD on the Primal: Coordinate Descent on the Dual}
\label{primal_dual_formulation_SGD_app}

In this section, we focus on the within-task algorithm we adopt in the paper,
namely \shortautoref{Within-Task Algorithm Online_paper}.
More precisely, we start from describing how the iterations generated by 
\shortautoref{Within-Task Algorithm Online_paper} can be considered as 
the primal iterations of a primal-dual algorithm in which the dual scheme
consists of a coordinate descent algorithm on the dual problem. After this,
we report in \shortautorefsubapp{main_inequality_sub} a key inequality for 
the dual decrease of this approach. From this result, 
a regret bound for \shortautoref{Within-Task Algorithm Online_paper}
and the proof of \shortautoref{epsilon-subgradient}, the key result describing 
the $\epsilon$-subgradients of our meta-algorithm, can be deduced as 
corollaries. This is done in \shortautorefsubapp{regret_bound_app} and 
\shortautorefsubapp{proof_prop_epsilon_subgradient_app}, respectively.

What follows is an adaptation of the theory developed in \cite{shalev2009mind},
where the authors do not emphasize the presence of the linear operator $\Xn$
and consider a slightly different dual problem. Specifically, proceeding as in 
\cite{shalev2009mind}, the primal-dual setting we need to consider is the following. 
At each iteration $k \in [n]$, we define the instantaneous primal problem 
\begin{equation} 
w_{h,k+1} = \argmin_{w \in \Real^d} ~ \Phi_{h,k+1}(w) 
\quad \quad \Phi_{h,k+1}(w) = \sum_{i=1}^{k} \ell_i \bigl( \langle x_i, w \rangle \bigr)
+ ~ \frac{k \la}{2} ~ \| w - h \|^2,
\end{equation}
where, $X_k \in \Real^{k \times d}$ is the matrix with rows only the 
first $k$ input vectors. The associated dual problem reads as follows
\begin{equation}
\tilde u_{h,k+1} = \argmin_ {\tilde u \in \Real^k} \Psi_{h,k+1} (\tilde u) \quad \quad 
\Psi_{h,k+1} (\tilde u) = \sum_{i =1}^k \ell_i^*(\tilde u_i) - \big \langle h, X_k \trans \tilde u \big \rangle
+ \frac{1}{2 k \la} \big \| X_k \trans \tilde u \big \|^2.
\end{equation}
In the following we will adopt the convention $\Phi_{h,1} \equiv \Psi_{h,1} \equiv 0$.

\begin{remark}[Strong Duality] \label{instantaneous_strong_duality}
Similarly to what observed in \shortautoref{approx_subgradient_prop}, 
also in this case, strong duality holds for each instantaneous 
couple of primal-dual problems above, namely, for any $k \in [n]$
\begin{equation} \label{strong_duality_k}
\Phi_{h, k+1} \bigl( w_{h, k+1} \bigr) = \min_{w \in \Real^d} \Phi_{h, k+1}(w) =
- \min_{\tilde u \in \Real^k} \Psi_{h, k+1}(\tilde u) =
- \Psi_{h, k+1} \bigl( \tilde u_{h, k+1} \bigr).
\end{equation}
Moreover, by the KKT conditions, we can express the primal solution by the 
dual one as follows
\begin{equation} \label{KKT_k}
w_{h, k+1} = - \frac{1}{k \la} X_k \trans \tilde u_{h, k+1} + h. 
\end{equation}
\end{remark}

\begin{remark}[Link Between the Instantaneous Problems and the Original Ones] \label{original_problems}
We observe that the original primal objective $\prhD$ in Eq. \eqref{primal_problem} and the 
corresponding dual objective $\duhD$ in Eq. \eqref{dual_problem} are respectively linked with 
the above instantaneous primal and dual objective functions in the following way
\begin{equation}
\frac{1}{n} \Phi_{h,n+1}(w) = \prhD (w), \quad \quad \quad 
\frac{1}{n} \Psi_{h,n+1} (\tilde u) = \duhD \Bigl( \frac{\tilde u}{n} \Bigr),
\end{equation}
for any $w \in \Real^d$ and any $\tilde u \in \Real^n$.
\end{remark}

\begin{algorithm}[H]
\caption{Within-Task Algorithm, Primal-Dual Version}\label{Within-Task Algorithm Online_new}
\begin{algorithmic}
\State ~
   \State {\bfseries Input} ~~ $\lambda > 0$ regularization parameter, $h \in \Real^d$ bias
   \vspace{.2cm}
   \State {\bfseries Initialization} ~~ $\tilde u_h^{(1)} = 0 \in \Real$, $w_h^{(1)} = h \in \Real^d$
  \vspace{.2cm}
   \State {\bfseries For} ~~ $k=1$ to $n$
   \vspace{.1cm}
    \State \qquad Receive ~~ $\ell_{k,h}(\cdot) = \ell_k\bigl( \langle x_k, \cdot \rangle \bigr) 
    + \displaystyle \frac{\la}{2} \| \cdot - h \|^2$ 
    \vspace{.1cm}
    \State \qquad Pay ~~ $\ell_{k,h} \bigl( w_h^{(k)} \bigr)$
   \vspace{.1cm}
   \State \qquad Update ~~ $\tilde u_h^{(k+1)}$ according to Eq. \eqref{update_constraint_new}
 \vspace{.1cm}
   \State \qquad Define ~~ $w_h^{(k+1)} = - \displaystyle \frac{1}{k \la} X_{k} \trans \tilde u_h^{(k+1)} + h$
   \vspace{.2cm}
 \State {{\bfseries Return} ~~ $\bigl( \tilde u_h^{(k)} \bigr)_{k =1}^{n+1}$, $\bigl( w_h^{(k)} \bigr)_{k =1}^{n+1}$,
 $\bwhD = \displaystyle \frac{1}{n} \sum_{k = 1}^n w_h^{(k)}$} 
\State ~
\end{algorithmic}
\end{algorithm}

As described in \cite{shalev2009mind}, we apply the coordinate descent algorithm on
the instantaneous dual problem outlined in \shortautoref{Within-Task Algorithm Online_new}.
More specifically, at the iteration $k$, the algorithm adds a coordinate at the last $k$-th position 
of the dual variable $\tilde u_h^{(k)}$ in the following way 
\begin{equation} \label{update_constraint_new}
\tilde u_{h,i}^{(k+1)} = 
\begin{cases}
u'_{k} & \text{if $i = k$} \\
\tilde u_{h,i}^{(k)} & \text{if $i \in [k-1]$},
\end{cases}
\end{equation}
where, $u'_{k} \in \displaystyle \partial \ell_k \bigl( \langle x_k, w_h^{(k)} \rangle \bigr)$.
We stress again that $\tilde u_h^{(k+1)} \in \Real^k$ and $\tilde u_h^{(k)} \in \Real^{k-1}$.
The primal variable is then updated by the KKT conditions outlined in Eq. \eqref{KKT_k}
in \shortautorefrem{instantaneous_strong_duality}. In the next lemma, we show that,
in this way, we exactly retrieve the iterations $\bigl( w_h^{(k)} \bigr)_k$ generated by 
\shortautoref{Within-Task Algorithm Online_paper}, and, consequently, the notation
does not conflict with the one used in the main body.

\begin{lemma}\label{identification}
Let $w_h^{(k+1)}$ be the update of the primal variable in \shortautoref{Within-Task Algorithm Online_new}.
Then, introducing the subgradient 
\begin{equation}
s_k = x_k u'_{k} + \la \bigl( w_h^{(k)} - h \bigr) \in \partial \ell_{k,h} \bigl( w_h^{(k)} \bigr),
\end{equation}
we can rewrite
\begin{equation}
w_h^{(k+1)}  =  w_h^{(k)} - \frac{1}{k \la} s_k.
\end{equation}
Consequently, the primal iterations generated by \shortautoref{Within-Task Algorithm Online_new}
coincides with the iterations generated by \shortautoref{Within-Task Algorithm Online_paper} in the
paper.
\end{lemma}

\begin{proof} 
We start from observing that, for any $k \in [n]$, by definition, we have 
\begin{equation}
w_h^{(k+1)} = - \frac{1}{k \la} X_{k} \trans \tilde u_h^{(k+1)} + h.
\end{equation}
For $k = 1$ the statement holds, as a matter of fact, introducing the subgradient
$s_1 = x_1 \tilde u'_1 + \la \bigl( w_h^{(1)} - h \bigr) \in \partial \ell_{1,h} \bigl( w_h^{(1)} \bigr)$,
we can write
\begin{equation}
w_h^{(2)} = - \frac{1}{\la} x_1 \tilde u_1' + h= - \frac{1}{\la} \bigl( s_1 - \la \bigl( w_h^{(1)}-h \bigr) \bigr) + h
= w_h^{(1)} - \frac{1}{\la} s_1.
\end{equation}
Now, we show that the statement holds also for $k = 2, \dots, n$. 
Since $X_k \trans \tilde u_h^{(k+1)} = X_{k-1} \trans \tilde u_h^{(k)} + x_k u'_k$, 
recalling again the subgradient $s_k = x_k u'_{k} + \la \bigl( w_h^{(k)}-h \bigr) \in 
\partial \ell_{k,h} \bigl( w_h^{(k)} \bigr)$ of the regularized loss, we can write the following
\begin{equation}
\begin{split}
w_h^{(k+1)} & = - \frac{1}{k \la} X_{k} \trans \tilde u_h^{(k+1)} + h
= - \frac{1}{k \la} \Bigl( X_{k-1} \trans \tilde u_h^{(k)} + x_k u'_k  \Bigr) + h\\
& = \frac{(k-1) \la}{k \la} \Bigl( - \frac{1}{(k-1) \la} X_{k-1} \trans \tilde u_h^{(k)} \Bigr)
- \frac{x_k u'_k}{k \la} + h\\
& = \frac{(k-1) \la (w_h^{(k)} - h) - s_k + \la ( w_h^{(k)} - h)}{k \la} + h \\
& = \frac{k \la w_h^{(k)} - s_k}{k \la} = w_h^{(k)} - \frac{1}{k \la} s_k.
\end{split}
\end{equation}
where, in the fourth equality, we have exploited the definition of the primal iterates in 
\shortautoref{Within-Task Algorithm Online_new}.
\end{proof}


\subsection{Main Inequality on the Dual Decrease}
\label{main_inequality_sub}

The next proposition is a key tool in our analysis. It coincides with a combination 
of slightly different versions of Lemma $2$ and Thm. $1$ in \cite{shalev2009mind}.

\begin{proposition}[Dual Decrease of \shortautoref{Within-Task Algorithm Online_new}, 
{\cite[Lemma $2$ and Thm. $1$]{shalev2009mind}}] \label{dual_decrease}
Let $\bigl( \tilde u_h^{(k)} \bigr)_k$, $\bigl( w_h^{(k)} \bigr)_k$ be generated 
according to \shortautoref{Within-Task Algorithm Online_new} for a fixed bias of $h \in \H$
and a regularization parameter $\la > 0$. Then, under \shortautoref{bounded_inputs} 
and \shortautoref{lipschitz_loss}, we have that
\begin{equation*} \label{rate_online_Lipsc}
\Psi_{h, n+1} \bigl( \tilde u_h^{(n+1)} \bigr) - \Psi_{h, n+1} \bigl( \tilde u_{h, n+1} \bigr) 
\le - \Bigl(\sum_{k =1}^n \ell_{k,h} \bigl( w_h^{(k)} \bigr) - \Phi_{h, n+1}\bigl( w_{h, n+1} \bigr) \Bigr) 
+ \frac{2 \rx^2 L^2 \bigl( {\rm{log}}(n) + 1 \bigr)}{\la}.
\end{equation*}
\end{proposition}

\begin{proof}
For any $k \in [n]$, using the convention $\Psi_{h,1} \equiv 0$, define the dual decrease
\begin{equation}
\Delta_k = \Psi_{h,k+1} \bigl( \tilde u_h^{(k+1)} \bigr) - \Psi_{h,k} \bigl( \tilde u_h^{(k)} \bigr).
\end{equation}
Hence, thanks to the telescopic sum and again the assumption $\Psi_{h,1} \equiv 0$, 
we can write
\begin{equation} \label{j_new}
\Psi_{h,n+1} \bigl( \tilde u_h^{(n+1)} \bigr) = \sum_{k=1}^n \Delta_k + 
\Psi_{h,1} \bigl( \tilde u_h^{(1)} \bigr) = \sum_{k=1}^n \Delta_k.
\end{equation}
We now show that, for any $k \in [n]$, the following relation holds
\begin{equation} \label{relation}
\Delta_k = - \ell_{k,h} \bigl( w_h^{(k)} \bigr) + \frac{1}{2 k \la} \big \| x_k u'_{k} + \la \bigl( w_h^{(k)} - h \bigr) \big \|^2.
\end{equation}
We start from considering the case $k = 2, \dots, n$. In this case, thanks to the updating rule in 
Eq. \eqref{update_constraint_new}, the fact $X_k \trans \tilde u_h^{(k+1)} = X_{k-1} 
\trans \tilde u_h^{(k)} + x_k u'_k$ and the closed form of the dual objective, we have that
\begin{equation*} \label{jj}
\begin{split}
\Delta_k & = \Psi_{h,k+1} \bigl( \tilde u_h^{(k+1)} \bigr) - \Psi_{h,k} \bigl( \tilde u_h^{(k)} \bigr) \\
& = \sum_{i =1}^{k-1} \ell_i^*\bigl( \tilde u^{(k)}_{h,i} \bigr) + \ell_k^*\bigl( u'_{k} \bigr)
+ \frac{1}{2 k \la} \big \| X_{k-1} \trans \tilde u_h^{(k)} + x_k u'_{k} \big \|^2 
- \big \langle h, X_{k-1} \trans \tilde u_h^{(k)} + x_k u'_{k} \big \rangle \\
& \quad \quad \quad \quad \quad \quad \quad \quad \quad \quad \quad 
\quad \quad \quad \quad \quad \quad \quad \quad \quad \quad \quad  
- \sum_{i =1}^{k-1} \ell_i^*\bigl( \tilde u^{(k)}_{h,i} \bigr) 
- \frac{1}{2 (k-1) \la} \big \| X_{k-1} \trans \tilde u_h^{(k)} \big \|^2 
+ \big \langle h, X_{k-1} \trans \tilde u_h^{(k)} \big \rangle \\
& = \ell_k^*\bigl( u'_{k} \bigr) + \frac{1}{2 k \la} \big \| X_{k-1} \trans \tilde u_h^{(k)} + x_k u'_{k} \big \|^2 
- \big \langle h, x_k u'_{k} \big \rangle - \frac{1}{2 (k-1) \la} \big \| X_{k-1} \trans \tilde u_h^{(k)} \big \|^2 \\
& = \ell_k^*\bigl( u'_{k} \bigr) + \frac{1}{2 \la} \Bigl( \frac{1}{k} -  \frac{1}{k-1}\Bigr) \big \| X_{k-1} 
\trans \tilde u_h^{(k)} \big \|^2 + \frac{1}{2 k \la} \big \| x_k u'_{k} \big \|^2 + \Big \langle \frac{X_{k-1} \trans 
\tilde u_h^{(k)}}{k \la} - h, x_k u'_{k} \Big \rangle \\
& = \ell_k^*\bigl( u'_{k} \bigr) + \frac{1}{2 \la} \Bigl( \frac{1}{k} -  \frac{1}{k-1}\Bigr) \la^2 (k-1)^2 
\big \| w_h^{(k)} - h \big \|^2 + \frac{1}{2 k \la} \big \| x_k u'_{k} \big \|^2 - \Big \langle \frac{(k-1) \la}
{k \la} \bigl( w_h^{(k)}-h \bigr) + h, x_k u'_{k} \Big \rangle \\
& = \ell_k^*\bigl( u'_{k} \bigr) + \frac{\la}{2} \Bigl( \frac{1}{k} - 1 \Bigr) 
\big \| w_h^{(k)}- h\big \|^2 + \frac{1}{2 k \la} \big \| x_k u'_{k} \big \|^2 - \Big \langle 
\Bigl( 1 - \frac{1}{k}\Bigr) \bigl( w_h^{(k)} - h \bigr) + h, x_k u'_{k} \Big \rangle \\
& = \Bigl( \ell_k^*\bigl( u'_{k} \bigr) - \big \langle w_h^{(k)}, x_k u'_{k} \big \rangle - 
\frac{\la}{2} \big \| w_h^{(k)} - h \big \|^2 \Bigr) + \frac{1}{2 k \la} \Bigl( 
\la^2 \big \| w_h^{(k)}-h \big \|^2 + \big \| x_k u'_{k} \big \|^2 + 2 \big 
\langle \la \bigl(w_h^{(k)}-h \bigr), x_k u'_{k} \big \rangle \Bigr) \\
& = - \Bigl(\ell_k \bigl( \langle x_k, w_h^{(k)} \rangle \bigr) + \frac{\la}{2} \big \| w_h^{(k)}-h \big \|^2 \Bigr) 
+ \frac{1}{2 k \la} \big \| x_k u'_{k} + \la \bigl(w_h^{(k)}-h\bigr) \big \|^2 \\
&  = - \ell_{k,h} \bigl( w_h^{(k)} \bigr) + \frac{1}{2 k \la} \big \| x_k u'_{k} + \la \bigl( w_h^{(k)}-h \bigr) \big \|^2,
\end{split}
\end{equation*}
where, in the fifth equality we have used the definition of the primal variable 
$w_h^{(k)} = - \displaystyle \frac{1}{(k-1) \la} X_{k-1} \trans \tilde u_h^{(k)}+h$,
in the sixth equality we have used the relation 
\begin{equation}
\frac{1}{\la}\Bigl( \frac{1}{k} - \frac{1}{k-1}\Bigr) \la^2 (k-1)^2 = \la \Bigl( \frac{1}{k} - 1 \Bigr), 
\end{equation}
and, finally, in the eighth equality we have exploited the assumption $u_k' \in \partial \ell_k \bigl( \langle x_k, w_h^{(k)} \rangle \bigr)$, implying by Fenchel--Young equality
\begin{equation}
\ell_k^*\bigl( u'_{k} \bigr) - \big \langle w_h^{(k)}, x_k u'_{k} \big \rangle = - \ell_k \bigl( \langle x_k, w_h^{(k)} \rangle \bigr).
\end{equation}
We now observe that the above relation in Eq. \eqref{relation} holds also in the 
case $k =1$, as a matter of fact, by definition, since $\Psi_{h,1} \equiv 0$, we have 
\begin{equation}
\begin{split}
\Delta_1 & = \Psi_{h,2} \bigl( \tilde u_h^{(2)} \bigr) - \Psi_{h,1} \bigl( \tilde u_h^{(1)} \bigr) 
= \Psi_{h,2} \bigl( \tilde u_h^{(2)} \bigr) \\
& = \ell_1^*(u'_{1}) - \big \langle h, x_1 u'_{1} \big \rangle + \frac{1}{2 \la} \big \| x_1 u'_{1} \big \|^2 \\
& = \Bigl( \ell_1^*(u'_{1}) - \big \langle w_h^{(1)}, x_1 u'_{1} \big \rangle - \frac{\la \big \| w_h^{(1)} - h \big \|^2}{2} \Bigr)
+ \frac{1}{2 \la} \big \| x_1 u'_{1} + \la \bigl(w_h^{(1)}-h \bigr) \big \|^2 \\
& = - \Bigl(\ell_1 \bigl( \langle x_1, w_h^{(1)} \rangle \bigr) + \frac{\la}{2} \big \| w_h^{(1)}-h \big \|^2 \Bigr)  + \frac{1}{2 \la} \big \| x_1 u'_{1} + \la \bigl( w_h^{(1)}-h \bigr) \big \|^2 \\
& = - \ell_{1,h} \bigl(w_h^{(1)} \bigr) + \frac{1}{2 \la} \big \| x_1 u'_{1} + \la \bigl( w_h^{(1)}-h \bigr) \big \|^2,
\end{split}
\end{equation} 
where, in the fourth equality we have rewritten
\begin{equation}
\frac{1}{2 \la} \big \| x_1 u'_{1} \big \|^2 
= \frac{1}{2 \la} \big \| x_1 u'_{1} + \la \bigl( w_h^{(1)}-h \bigr) \big \|^2
- \big \langle w_h^{(1)}-h, x_1 u'_{1} \big \rangle - 
\frac{\la \big \| w_h^{(1)} - h \big \|^2}{2},
\end{equation}
and, in the fifth equality, we have used again the assumption 
$u_1' \in \partial \ell_1 \bigl( \langle x_1, w_h^{(1)} \rangle \bigr)$, implying 
by Fenchel--Young equality
\begin{equation}
\ell_1^*\bigl( u'_{1} \bigr) - \big \langle w_h^{(1)}, x_1 u'_{1} \big \rangle =
- \ell_1 \bigl( \langle x_1, w_h^{(1)} \rangle \bigr).
\end{equation}
Therefore, using Eq. \eqref{j_new} and summing over $k \in [n]$, we get the following
\begin{equation} \label{ggg}
\begin{split}
\Psi_{h,n+1} \bigl( \tilde u_h^{(n+1)} \bigr) & = \sum_{k =1}^n \Delta_k = - \sum_{k =1}^n 
\ell_{k,h} \bigl( w_h^{(k)} \bigr) + \sum_{k =1}^n \frac{1}{2 k \la} \big \| x_k u'_{k} + 
\la \bigl( w_h^{(k)}-h \bigr) \big \|^2 \\
& = - \sum_{k =1}^n \ell_{k,h} \bigl( w_h^{(k)} \bigr) + \frac{1}{2 \la} \sum_{k =1}^n \frac{1}{k} 
\big \| x_k u'_{k} + \la \bigl( w_h^{(k)} - h \bigr) \big \|^2.
\end{split}
\end{equation}
Now, for $k = 2, \dots, n$, thanks to the definition of $w_h^{(k)}$, we can write
\begin{equation}
\la \bigl(w_h^{(k)}-h \bigr) = - \frac{1}{k-1} X_{k-1} \trans \tilde u_h^{(k)} 
= - \frac{1}{k-1} \sum_{i = 1}^{k-1} x_i u'_{i}.
\end{equation}
Hence, under \shortautoref{bounded_inputs} and \shortautoref{lipschitz_loss},
since $ | u'_i | \le L $, for any $i$, for $k = 2, \dots, n$, we get
\begin{equation}
\big \| \la \bigl( w_h^{(k)}-h \bigr) \big \| \le L \rx.
\end{equation}
Moreover, we observe that the above majorization holds also for the case $k = 1$, 
as a matter of fact, thanks to the definition $w_h^{(1)} = h$, we have that
\begin{equation}
\big \| \la \bigl( w_h^{(1)}-h \bigr) \big \| = 0.
\end{equation}
Hence, using the inequality $\| a + b \|^2 \le 2 \| a \|^2 + 2 \| b \|^2$ for any $a, b \in \Real^d$, 
for $k = 1, \dots, n$, we get
\begin{equation}
\big \| x_k u'_{k} + \la \bigl( w_h^{(k)}-h \bigr) \|^2 \le 
2 \big \| x_k u'_{k} \big \|^2 + 2 \big \| \la \bigl( w_h^{(k)}-h \bigr) \big \|^2 \le 4 \rx^2 L^2.
\end{equation}
Finally, coming back to Eq. \eqref{ggg}, using the inequality $\sum_{k =1}^n 1/k \le 
{\rm{log}}(n) + 1$, we get
\begin{equation}
\begin{split}
\Psi_{h,n+1} \bigl( \tilde u_h^{(n+1)} \bigr) & \le - \sum_{k =1}^n \ell_{k,h} \bigl( w_h^{(k)} \bigr) 
+ \frac{1}{2 \la} \sum_{k =1}^n \frac{1}{k} \big \| x_k u'_{k} + \la \bigl( w_h^{(k)}-h \bigr) \big \|^2 \\
& \le - \sum_{k =1}^n \ell_{k,h} \bigl( w_h^{(k)} \bigr) + \frac{2 \rx^2 L^2}{\la} \sum_{k =1}^n \frac{1}{k} \\
& \le - \sum_{k =1}^n \ell_{k,h} \bigl( w_h^{(k)} \bigr) + \frac{2 \rx^2 L^2 \bigl( {\rm{log}}(n) + 1 \bigr)}{\la}.
\end{split}
\end{equation}
The desired statement follows by adding to both sides $- \Psi_{h, n+1} \bigl( \tilde u_{h, n+1} \bigr)$ 
and observing that, by strong duality, as already observed in Eq. \eqref{strong_duality_k} in 
\shortautorefrem{instantaneous_strong_duality}, we have that
\begin{equation}
\Phi_{h, k+1} \bigl( w_{h, k+1} \bigr) = \min_{w \in \Real^d} \Phi_{h, k+1}(w) =
- \min_{\tilde u \in \Real^k} \Psi_{h, k+1}(\tilde u) =
- \Psi_{h, k+1} \bigl( \tilde u_{h, k+1} \bigr).
\end{equation}
\end{proof}


\subsection{Regret Bound for \shortautoref{Within-Task Algorithm Online_new}}
\label{regret_bound_app}

The following result is a direct corollary of \shortautoref{dual_decrease}. 
It is a well-known fact and it coincides with a regret bound for the iterations in 
\shortautoref{Within-Task Algorithm Online_paper}. This result will be then
used in the following \shortautorefsubapp{statistical_preliminaries_app} 
in order to prove \shortautoref{estimation_error}.

\begin{corollary} \label{inner regret bound}
Let $\bigl( w_h^{(k)} \bigr)_k$ be the iterations generated by \shortautoref{Within-Task Algorithm Online_paper}.
Then, under the same assumptions of \shortautoref{dual_decrease}, for any $w \in \Real^d$, 
the following regret bound holds
\begin{equation}
\frac{1}{n} \sum_{k =1}^n \ell_{k,h} \bigl( w_h^{(k)} \bigr) - \prhD( w ) \le 
\frac{1}{n} \sum_{k =1}^n \ell_{k,h} \bigl( w_h^{(k)} \bigr) - \prhD( \whD )
\le \frac{2 \rx^2 L^2 \bigl( {\rm{log}}(n) + 1 \bigr)}{\la n}.
\end{equation}
\end{corollary}

\begin{proof} 
We start from observing that, as already pointed out in \shortautorefrem{original_problems}, 
for any $w \in \Real^d$, we have $\Phi_{h,n+1}(w) / n = \prhD (w)$. 
Consequently, $\whD = \argmin_{w \in \Real^d} \Phi(w) = \argmin_{w 
\in \Real^d} \Phi_{h, n+1} (w) = w_{h, n+1}$. This implies $\Phi_{h, n+1}
\bigl( w_{h, n+1} \bigr)/n = \prhD( \whD)$. Hence, thanks to this last observation, the 
definition of $\tilde u_{h, n+1}$ and \shortautoref{dual_decrease}, we can write
\begin{equation*}
0 \le \Psi_{h, n+1} \bigl( \tilde u_h^{(n+1)} \bigr) - \Psi_{h, n+1} \bigl( \tilde u_{h, n+1} \bigr) 
\le - \Bigl(\sum_{k =1}^n \ell_{k,h} \bigl( w_h^{(k)} \bigr) - \Phi_{h, n+1}\bigl( w_{h, n+1} \bigr) \Bigr) 
+ \frac{2 \rx^2 L^2 \bigl( {\rm{log}}(n) + 1 \bigr)}{\la}.
\end{equation*}
The statement derives from dividing by $n$. The first inequality simply derives from the definition of $\whD$.
\end{proof}


\subsection{Proof of \shortautoref{epsilon-subgradient}}
\label{proof_prop_epsilon_subgradient_app}

The second corollary deriving from \shortautoref{dual_decrease} is the main tool used to prove
\shortautoref{epsilon-subgradient}. It essentially states that the last dual iteration of 
\shortautoref{Within-Task Algorithm Online_new} is an $\epsilon$-minimizer of our original
dual objective $\duhD$ in Eq. \eqref{dual_problem}, for an appropriate value of $\epsilon$. 
This observation, combined with an expectation argument and \shortautoref{approx_subgradient_prop}, 
allows us to build an $\epsilon$-subgradient for the meta-objective function, as described in 
\shortautoref{epsilon-subgradient}. 

\begin{corollary} \label{dual_accuracy}
Let $\tilde u_h^{(n+1)}$ be the last dual iteration of \shortautoref{Within-Task Algorithm Online_new}.
Then, under the same assumptions of \shortautoref{dual_decrease}, for any $w \in \Real^d$, 
the vector $\hat u_h = \tilde u_h^{(n+1)}/n$ is an $\epsilon$-minimizer of the dual objective
$\duhD$ in Eq. \eqref{dual_problem}, with
\begin{equation} \label{epsilon_our_v}
\epsilon = - \Bigl(\frac{1}{n} \sum_{k =1}^n \ell_{k,h} \bigl( w_h^{(k)} \bigr) - \prhD \bigl( w \bigr) \Bigr) 
+ \frac{2 \rx^2 L^2 \bigl( {\rm{log}}(n) + 1 \bigr)}{\la n},
\end{equation}
where $\bigl( w_h^{(k)} \bigr)_k$ is the iteration generated by \shortautoref{Within-Task Algorithm Online_paper}.
\end{corollary}

\begin{proof}
We start from recalling that, as already observed in \shortautoref{dual_decrease},
the primal iterations generated by \shortautoref{Within-Task Algorithm Online_new} 
coincide with the iterations generated by \shortautoref{Within-Task Algorithm Online_paper}.
Now, thanks to \shortautoref{dual_decrease}, dividing by $n$, we have that
\begin{equation*}
\frac{1}{n} \Psi_{h, n+1} \bigl( \tilde u_h^{(n+1)} \bigr) - \frac{1}{n} \Psi_{h, n+1} \bigl( \tilde u_{h, n+1} \bigr) 
\le \tilde \epsilon,
\end{equation*}
with 
\begin{equation}
\tilde \epsilon = - \Bigl(\frac{1}{n} \sum_{k =1}^n \ell_{k,h} \bigl( w_h^{(k)} \bigr) - \frac{1}{n} \Phi_{h, n+1}
\bigl( w_{h, n+1} \bigr) \Bigr) + \frac{2 \rx^2 L^2 \bigl( {\rm{log}}(n) + 1 \bigr)}{\la n}.
\end{equation}
As already pointed out, we now observe that, $\Phi_{h, n+1} \bigl( w_{h, n+1} \bigr)/n = 
\prhD( \whD)$, hence, for any $w \in \Real^d$, we can rewrite 
\begin{equation} \label{epsilon_value}
\begin{split}
\tilde \epsilon & = - \Bigl(\frac{1}{n} \sum_{k =1}^n \ell_{k,h} \bigl( w_h^{(k)} \bigr) - \prhD( \whD) \Bigr) 
+ \frac{2 \rx^2 L^2 \bigl( {\rm{log}}(n) + 1 \bigr)}{\la n} \\ & \le 
- \Bigl(\frac{1}{n} \sum_{k =1}^n \ell_{k,h} \bigl( w_h^{(k)} \bigr) - \prhD( w ) \Bigr) 
+ \frac{2 \rx^2 L^2 \bigl( {\rm{log}}(n) + 1 \bigr)}{\la n} = \epsilon.
\end{split}
\end{equation}
Summarizing, we have obtained that 
\begin{equation} \label{rrr}
0 \in \partial_\epsilon \Bigl( \frac{1}{n}\Psi_{h, n+1} \Bigr)(\tilde u_h^{(n+1)}), 
\end{equation}
where the value of $\epsilon$ is the one in Eq. \eqref{epsilon_value}. Now, we observe that, 
thanks to the relation $\displaystyle \frac{1}{n} \Psi_{h,n+1} (\tilde u) = \duhD \Bigl( \frac{\tilde u}{n} \Bigr)
= \Bigl( \duhD \circ \frac{1}{n} \Bigr) (\tilde u)$ (see \shortautorefrem{original_problems}), exploiting \shortautoref{epsilon_subgradient_chain_rule_partialscalar}, for any $\tilde u \in \Real^n$, we have that, 
\begin{equation}
\partial_\epsilon \Bigl( \frac{1}{n}\Psi_{h, n+1} \Bigr)(\tilde u)
= \partial_\epsilon \Bigl( \duhD \circ \frac{1}{n} \Bigr) (\tilde u)
= \frac{1}{n} \partial_\epsilon \duhD \Bigl( \frac{\tilde u}{n} \Bigr).
\end{equation}
Consequently, Eq. \eqref{rrr}, implies $\displaystyle 0 \in \partial_\epsilon \duhD \Bigl( \frac{\tilde u_h^{(n+1)}}{n} \Bigr)$,
which is equivalent, as already observed in \shortautoref{epsilon_minimum}, to the desired statement.
\end{proof}

The last ingredient we need to prove \shortautoref{epsilon-subgradient} is the following
expectation argument.

\begin{corollary} \label{epsilon_accuracy_in_expectation}
Let $\bigl( w_h^{(k)} \bigr)_k$ be the iterations generated by \shortautoref{Within-Task Algorithm Online_paper},
Let $\epsilon$ be the value in \shortautoref{dual_accuracy} with $w = w_{\task, h}$, where $\displaystyle w_{\task, h} = 
\argmin_{w \in \Real^d} \cR_\task \bigl( w \bigr) + \frac{\la}{2} \big \| w - h \big \|^2$. Then, 
under the same assumptions of \shortautoref{dual_decrease}, we have that
\begin{equation}
\Exp_{\Zn \sim \task^n} \big[ \epsilon \big] \le \frac{2 \rx^2 L^2 \bigl( {\rm{log}}(n) + 1 \bigr)}{\la n}.
\end{equation}
\end{corollary}

\begin{proof} 
We recall that the value of $\epsilon$ in \shortautoref{dual_accuracy} with $w = w_{\task, h}$ is
explicitly given by
\begin{equation}
\epsilon = - \Bigl( \frac{1}{n} \sum_{k =1}^n \ell_k \bigl( \langle x_k, w_h^{(k)} \rangle \bigr) 
+ \frac{\la}{2} \big \| w_h^{(k)}-h \big \|^2 - \prhD(w_{\task, h}) \Bigr) 
+ \frac{2 \rx^2 L^2 \bigl( {\rm{log}}(n) + 1 \bigr)}{\la n}.
\end{equation}
Hence, to prove the statement we just need to show that
\begin{equation}
0 \le \Exp_{\Zn \sim \task^n} \Big[ \frac{1}{n} \sum_{k =1}^n \ell_k \bigl( \langle x_k, w_h^{(k)} \rangle \bigr) 
+ \frac{\la}{2} \big \| w_h^{(k)}-h \big \|^2 - \prhD(w_{\task, h}) \Big].
\end{equation}
In order to do this, we recall that $\bwhD$ denotes the average of the first $n$ iterations 
$\bigl( w_h^{(k)} \bigr)_k$ and we observe the following
\begin{equation}
\begin{split}
0 & \le \Exp_{\Zn \sim \task^n}~\Big[ \cR_\task \bigl( \bwhD(\Zn) \bigr) + \frac{\la}{2} \big \| \bwhD(\Zn)-h \big \|^2 \Big]
- \Exp_{\Zn \sim \task^n}~\Big[  \cR_\task \bigl( w_{\task,h} \bigr) + \frac{\la}{2} \big \| w_{\task,h}-h \big \|^2 \Big] \\
& \le \Exp_{\Zn \sim \task^n}~\Big[ \frac{1}{n} \sum_{i =1}^n \cR_\task \bigl( \whD^{(i)} \bigr) 
+ \frac{\la}{2} \big \| \whD^{(i)}-h \big \|^2 \Big] - \Exp_{\Zn \sim \task^n}~\Big[  \cR_\task \bigl( w_{\task,h} \bigr) 
+ \frac{\la}{2} \big \| w_{\task,h}-h \big \|^2 \Big] \\
& = \Exp_{\Zn \sim \task^n}~\Big[ \frac{1}{n} \sum_{i =1}^n \cR_\task \bigl( \whD^{(i)} \bigr) 
+ \frac{\la}{2} \big \| \whD^{(i)}-h \big \|^2 \Big] - \Exp_{\Zn \sim \task^n}~\Big[  \frac{1}{n} \sum_{i =1}^n 
\ell_i \bigl( \langle x_i, w_{\task, h} \rangle \bigr) + \frac{\la}{2} \big \| w_{\task,h}-h \big \|^2 \Big] \\
& = \Exp_{\Zn \sim \task^n}~\Big[ \frac{1}{n} \sum_{i =1}^n \ell_i \bigl( \langle x_i, \whD^{(i)} \rangle \bigr)
+ \frac{\la}{2} \big \| \whD^{(i)}-h \big \|^2 \Big] - \Exp_{\Zn \sim \task^n}~\Big[  \frac{1}{n} \sum_{i =1}^n 
\ell_i \bigl( \langle x_i, w_{\task, h} \rangle \bigr) + \frac{\la}{2} \big \| w_{\task,h}-h \big \|^2 \Big] \\
& = \Exp_{\Zn \sim \task^n} \Big[ \frac{1}{n} \sum_{k =1}^n \ell_k \bigl( \langle x_k, w_h^{(k)} \rangle \bigr) 
+ \frac{\la}{2} \big \| w_h^{(k)} - h \big \|^2 - \prhD(w_{\task,h}) \Big],
\end{split}
\end{equation}
where, the first inequality is a consequence of the definition of $w_{\task,h}$, 
the second inequality derives from Jensen's inequality, the first equality holds
since $w_{\task, h}$ does not depend on the data and, finally, the second equality
holds by standard online-to-batch arguments, more precisely, since $\whD^{(i)}$ 
does not depend on the point $z_i$, we have that, $\Exp_{\Zn \sim \task^n} 
\ell_i \bigl( \langle x_i, \whD^{(i)} \rangle \bigr) = \Exp_{\Zn \sim \task^n} 
\cR_\task \bigl( \whD^{(i)} \bigr)$.
\end{proof}

We now are ready to prove \shortautoref{epsilon-subgradient}.

\epsilonsubgradient*

\begin{proof}
We start from observing that, thanks to \shortautoref{identification},
we have that 
\begin{equation} \label{hhhh}
-\la \bigl( \whD^{(n+1)} - h \bigr) = \Xn \trans \frac{\tilde u_h^{(n+1)}}{n}.
\end{equation}
Hence, thanks to \shortautoref{approx_subgradient_prop} and 
\shortautoref{dual_accuracy} applied to the vector $\hat u_h = \tilde 
u_h^{(n+1)}/n$, we can state that $-\la \bigl( \whD^{(n+1)} - h \bigr) \in 
\partial_\epsilon \LL_{\Zn}(h)$, with $\epsilon$ given as in Eq. 
\eqref{epsilon_our_v}, choosing $w = w_{\task,h}$. Hence, the statement 
in Eq. \eqref{pp} is a consequence of these observations and 
\shortautoref{epsilon_accuracy_in_expectation}.
Finally, we observe that, thanks to the fact $-\la \bigl( \whD^{(n+1)} - h \bigr) \in 
\partial_\epsilon \LL_{\Zn}(h)$ and Eq. \eqref{discrepancy_gradients}
in \shortautorefex{epsilon_subgradient_Moreau_envelope}, we know that
\begin{equation}
\big \| \nabla \LL_{\Zn}(h) - \hat \nabla \LL_{\Zn}(h) \big \|^2 \le 2 \la \epsilon,
\end{equation}
where $\epsilon$ is the same value as before. The statement in Eq. \eqref{ppp} 
derives from taking the expectation with respect to the dataset $\Zn$ and applying again
the result in \shortautoref{epsilon_accuracy_in_expectation}.
\end{proof}


\section{Proof of \shortautoref{estimation_error}}
\label{statistical_preliminaries_app}

In this section, we report the proof of \shortautoref{estimation_error} which is often 
used in the main body of this work. The proof exploits the regret bound for \shortautoref{Within-Task Algorithm Online_paper}
given in \shortautoref{inner regret bound} in \shortautorefsubapp{regret_bound_app} and it essentially relies on 
online-to-batch conversion arguments.

\estimation*

\begin{proof}
The proof is similar to the one of \shortautoref{epsilon_accuracy_in_expectation}. 
More precisely, we can write
\begin{equation*}
\begin{split}
\Exp&_{\Zn \sim \task^n}~\Big[ \cR_\task \bigl( \bwhD(\Zn) \bigr) \Big]
- \Exp_{\Zn \sim \task^n}~\Big[ \cR_{\Zn} \bigl( \whD (\Zn) \bigr) + \frac{\la}{2} \big \| \whD(\Zn) - h \big \|^2 \Big] \\
& \le \Exp_{\Zn \sim \task^n}~\Big[ \cR_\task \bigl( \bwhD(\Zn) \bigr) + \frac{\la}{2} \big \| \bwhD(\Zn) - h \big \|^2 \Big]
- \Exp_{\Zn \sim \task^n}~\Big[  \cR_{\Zn} \bigl( \whD (\Zn) \bigr) + \frac{\la}{2} \big \| \whD(\Zn) - h \big \|^2 \Big] \\
& \le \Exp_{\Zn \sim \task^n}~\Big[ \frac{1}{n} \sum_{i =1}^n \cR_\task \bigl( \whD^{(i)} \bigr) 
+ \frac{\la}{2} \big \| \whD^{(i)} - h \big \|^2 \Big] - \Exp_{\Zn \sim \task^n}~\Big[  \cR_{\Zn} \bigl( \whD (\Zn) \bigr) 
+ \frac{\la}{2} \big \| \whD(\Zn) - h \big \|^2 \Big] \\
& = \Exp_{\Zn \sim \task^n}~\Big[ \frac{1}{n} \sum_{i =1}^n \ell_i \bigl( \langle x_i, \whD^{(i)} \rangle \bigr) 
+ \frac{\la}{2} \big \| \whD^{(i)} - h \big \|^2 \Big] - \Exp_{\Zn \sim \task^n}~\Big[  \cR_{\Zn} \bigl( \whD (\Zn) \bigr) 
+ \frac{\la}{2} \big \| \whD(\Zn) - h \big \|^2 \Big] \\
& = \Exp_{\Zn \sim \task^n}~\Big[ \frac{1}{n} \sum_{i =1}^n \ell_i \bigl( \langle x_i, \whD^{(i)} \rangle \bigr) 
+ \frac{\la}{2} \big \| \whD^{(i)}- h \big \|^2 - \prhD(\whD(\Zn)) \Big] \\
& \le \frac{2 \rx^2 L^2 \bigl( {\rm{log}}(n) + 1 \bigr)}{\la n},
\end{split}
\end{equation*}
where, in the first inequality we have exploited the non-negativity of the regularizer 
and in the second inequality we have applied Jensen's inequality. The first equality above holds
by standard online-to-batch arguments, more precisely, since $\whD^{(i)}$ does not depend on 
the point $z_i$, we have that, $\Exp_{\Zn \sim \task^n} \ell_i \bigl( \langle x_i, \whD^{(i)} 
\rangle \bigr) = \Exp_{\Zn \sim \task^n} \cR_\task \bigl( \whD^{(i)} \bigr)$. Finally, the last inequality 
is due to the application of the regret bound given in \shortautoref{inner regret bound}.
\end{proof}


\section{Convergence Rate of \shortautoref{OGDA2_paper}}
\label{estimated_h_generalization_bound_proof_app}

In this section, we give the convergence rate bound of \shortautoref{OGDA2_paper} which is used 
in the paper for the proof of the excess transfer risk bound given in \shortautoref{excess_transfer_risk_LTL}.

We recall that the meta-algorithm we adopt to estimate the bias $h$ is SGD applied to the 
function $\hat \E_n(\cdot) = \Exp_{\task \sim \env}~\Exp_{\Zn \sim \task^n}~ \LL_{\Zn}(\cdot)$. 
We recall also that, at each iteration $t$, the meta-algorithm approximate the gradient of the function $\LL_t$
at the point $h^{(t)}$ by the vector $\hat \nabla_t$ which is computed as described in Eq. \eqref{approx_meta_gradient_t}.
In the subsequent analysis we use the notation
\begin{equation}
\Exp_{\Zn^{(t)}} ~ \big[ \cdot \big] = \Exp ~ \Big[ \cdot \big | \Zn^{(1)}, \dots, \Zn^{(t-1)} \Big],
\end{equation}
where, the expectation must be intended with respect to the sampling of the dataset from the distribution 
induced by the sampling of the task $\task \sim \env$ and then the sampling of the dataset from that task.
We observe that, thanks to \shortautoref{epsilon-subgradient} and the independence of $h^{(t)}$ on
$\Zn^{(t)}$, we can state that this vector $\hat \nabla_t$ is an $\epsilon_t$-subgradient of $\LL_t$ at 
the point $h^{(t)}$, where, $\epsilon_t$ is such that
\begin{equation} \label{uuuu}
\Exp_{\Zn^{(t)}} ~ \big[ \epsilon_t \big] \le \frac{2 \rx^2 L^2 \bigl( {\rm{log}}(n) + 1 \bigr)}{\la n}
\end{equation}
\begin{equation} \label{jjjjj}
\Exp_{\Zn^{(t)}} ~ \big \| \nabla^{(t)} - \hat \nabla^{(t)} \big \|^2
\le \frac{4 \rx^2 L^2 \bigl( {\rm{log}}(n) + 1 \bigr)}{n}.
\end{equation}
Before proceeding with the proof of the convergence rate of 
\shortautoref{OGDA2_paper}, we need to introduce the 
following result contained in \cite{shalev2014understanding}. 

\begin{lemma}[See {\cite[Lemma $14.1$]{shalev2014understanding}}] \label{basic_lemma_1_proj}
Let $h^{(t)}$ be update of \shortautoref{OGDA2_paper}. Then, for any $\hat h \in \H$, we have 
\begin{equation*}
\sum_{t=1}^T \big \langle h^{(t)} - \hat h, \hat \nabla^{(t)} \big \rangle \le \frac{1}{2} \Bigl( \frac{1}{\gamma} 
\big \| h^{(1)} - \hat h \big \|^2 + \gamma \sum_{t =1}^T \big \| \hat \nabla^{(t)} \big \|^2 \Bigr).
\end{equation*}
\end{lemma}

\begin{proof}
Thanks to the definition of the update, for any $\hat h \in \H$, we have that
\begin{equation*}
\begin{split}
\big \| h^{(t+1)} - \hat h \big \|^2 \le \big \| h^{(t)} - \gamma \hat \nabla^{(t)} - \hat h \big \|^2 
= \big \| h^{(t)} - \hat h \big \|^2 - 2 \gamma \big \langle h^{(t)} - 
\hat h, \hat \nabla^{(t)} \big \rangle + \gamma^2 \big \| \hat \nabla^{(t)} \big \|^2.
\end{split}
\end{equation*}
Hence, rearranging the terms, we get the following
\begin{equation*}
\begin{split}
\big \langle h^{(t)} - \hat h, \hat \nabla^{(t)} \big \rangle = \frac{1}{2 \gamma} \Bigl( \big \| h^{(t)} - \hat h \big \|^2 -
\big \| h^{(t+1)} - \hat h \big \|^2 \Bigr) + \frac{\gamma}{2} \big \| \hat \nabla^{(t)} \big \|^2.
\end{split}
\end{equation*}
Summing over $t \in [T]$, exploiting the telescopic sum and the fact 
$- \big \| h^{(T+1)} - \hat h \big \|^2 \le 0$, the statement follows.  
\end{proof}

We now are ready to study the convergence rate of \shortautoref{OGDA2_paper}.

\begin{restatable}[Convergence Rate of \shortautoref{OGDA2_paper}]{proposition}{regret} \label{regret_OGDA_proj}
Let \shortautoref{bounded_inputs} and \shortautoref{lipschitz_loss} hold and
let $\bar h_T$ be the output of \shortautoref{OGDA2_paper} run with step size
\begin{equation}
\gamma = \frac{\sqrt{2} ~ \big \| \hat h \big \|}{L \rx} ~ \sqrt{
\Bigl(T \Bigl( 1 + \frac{4 \bigl( {\rm{log}}(n) + 1 \bigr)}{n} \Bigr) \Bigr)^{-1}}. 
\end{equation}
Then, for any $\hat h \in \H$, we have that
\begin{equation*}
\Exp ~ \hat \E_n(\bar h_T) - \hat \E_n(\hat h)
\le \big \| \hat h \big \| ~ L \rx ~ \sqrt{2 \Bigl( 1 + \frac{4 \bigl( {\rm{log}}(n) + 1 \bigr)}{n} \Bigr) \frac{1}{T}}
+ \frac{2 \rx^2 L^2 \bigl( {\rm{log}}(n) + 1 \bigr)}{\la n},
\end{equation*}
where, the expectation above is with respect to the sampling 
of the datasets $\Zn^{(1)}, \dots, \Zn^{(T)}$ from the environment $\env$.
\end{restatable}

\begin{proof}
We start from observing that, by convexity of $\LL_{\Zn^{(t)}}$, thanks to
the fact that $\hat \nabla_t$ is an $\epsilon_t$-subgradient of $\LL_{\Zn^{(t)}}$ 
at the point $h^{(t)}$, for any $\hat h \in \H$, we can write
\begin{equation}
\LL_{\Zn^{(t)}} \bigl( h^{(t)} \bigr) - \LL_{\Zn^{(t)}} \bigl( \hat{h} \bigr)
\le \big \langle \hat \nabla^{(t)}, h^{(t)} - \hat{h} \big \rangle + \epsilon_t.
\end{equation}
Now, taking the expectation with respect to the sampling of $\Zn^{(t)}$, 
thanks to what observed in Eq. \eqref{uuuu}, we have 
\begin{equation} 
\begin{split} 
\Exp_{\Zn^{(t)}} \Big[ \LL_{\Zn^{(t)}} \bigl( h^{(t)} \bigr) - \LL_{\Zn^{(t)}} \bigl( \hat{h} \bigr) \Big]
& \le \Exp_{\Zn^{(t)}} \big \langle \hat \nabla^{(t)}, h^{(t)} - \hat{h} \big \rangle + 
\underbrace{\Exp_{\Zn^{(t)}} \big[ \epsilon_t \big]}_{} \\
& \le \Exp_{\Zn^{(t)}} \big \langle \hat \nabla^{(t)}, h^{(t)} - \hat{h} \big \rangle + 
\underbrace{\frac{2 \rx^2 L^2 \bigl( {\rm{log}}(n) + 1 \bigr)}{\la n}}_{\displaystyle \epsilon_{\la,n}}.
\end{split}
\end{equation}
Hence, taking the global expectation, we get
\begin{equation}  
\Exp ~ \Big[ \LL_{\Zn^{(t)}} \bigl( h^{(t)} \bigr) - \LL_{\Zn^{(t)}} \bigl( \hat{h} \bigr) \Big]
\le \Exp ~ \big \langle \hat \nabla^{(t)}, h^{(t)} - \hat{h} \big \rangle + \epsilon_{\la,n}.
\end{equation}
Summing over $t \in [T]$ and dividing by $T$, we get
\begin{equation} \label{iii}
\frac{1}{T} \sum_{t = 1}^T \Exp ~ \Big[ \LL_{\Zn^{(t)}} \bigl( h^{(t)} \bigr) - 
\LL_{\Zn^{(t)}} \bigl( \hat{h} \bigr) \Big] \le \frac{1}{T} \sum_{t = 1}^T \Exp ~
\big \langle \hat \nabla^{(t)}, h^{(t)} - \hat{h} \big \rangle + \epsilon_{\la,n}.
\end{equation}
Now, applying \shortautoref{basic_lemma_1_proj}, as regards the first term 
of the RHS in the bound above, we can write
\begin{equation} \label{kkoo_proj}
\begin{split}
\frac{1}{T} \sum_{t = 1}^T \Exp ~ \big \langle \hat \nabla^{(t)}, h^{(t)} - \hat{h} \big \rangle
 & \le \frac{1}{2} \Bigl( \frac{1}{\gamma T} \big \| h^{(1)} - \hat h \big \|^2 + \frac{\gamma}{T} 
 \sum_{t =1}^T \Exp ~ \big \| \hat \nabla^{(t)} \big \|^2 \Bigr).
\end{split}
\end{equation}
Now we observe that, thanks to \shortautoref{bounded_inputs}, \shortautoref{lipschitz_loss} 
and \shortautoref{properties}, $\big \| \nabla^{(t)} \big \|
\le \rx L$ for any $t \in [T]$. Consequently, using the inequality $\| a + b \|^2 \le 
2 \| a \|^2 + 2 \| b \|^2$ for any two vectors $a, b \in \Real^d$ and applying Eq. \eqref{jjjjj},
we can write the following
\begin{equation}
\begin{split}
\Exp_{\Zn^{(t)}} ~ \big \| \hat \nabla^{(t)} \big \|^2 & = 
\Exp_{\Zn^{(t)}} ~ \big \| \hat \nabla^{(t)} \pm \nabla^{(t)} \big \|^2 
\le 2 ~ \Exp_{\Zn^{(t)}} ~ \big \| \nabla^{(t)} \big \|^2 + 
2 ~ \Exp_{\Zn^{(t)}} ~ \big \| \nabla^{(t)} - \hat \nabla^{(t)} \big \|^2 \\
& \le 2 L^2 \rx^2 + \frac{8 \rx^2 L^2 \bigl( {\rm{log}}(n) + 1 \bigr)}{n}
= 2 L^2 \rx^2 \Bigl( 1 + \frac{4 \bigl( {\rm{log}}(n) + 1 \bigr)}{n} \Bigr).
\end{split}
\end{equation}
Hence, taking the global expectation of the above relation and 
combining with Eq. \eqref{kkoo_proj}, we get
\begin{equation} \label{oo1}
\frac{1}{T} \sum_{t = 1}^T \Exp ~ \big \langle \hat \nabla^{(t)}, h^{(t)} - \hat{h} \big \rangle
\le \frac{1}{2} \Bigl( \frac{1}{\gamma T} \big \| h^{(1)} - \hat h \big \|^2 + 
2 L^2 \rx^2 \Bigl( 1 + \frac{4 \bigl( {\rm{log}}(n) + 1 \bigr)}{n} \Bigr) \gamma \Bigr).
\end{equation}
We now observe that, as regards the LHS member in Eq. \eqref{iii}, by Jensen's 
inequality and the independence of $h^{(t)}$ on $\Zn^{(t)}$, we have that
\begin{equation} \label{oo2}
\Exp ~ \hat \E_n(\bar h_T) - \hat \E_n(\hat h) =
\Exp ~ \Big[ \LL_{\Zn}(\bar h_T) - \LL_{\Zn}(\hat h) \Big]
\le \frac{1}{T} \sum_{t = 1}^T \Exp ~ \Big[ \LL_{\Zn^{(t)}} \bigl( h^{(t)} \bigr) - 
\LL_{\Zn^{(t)}} \bigl( \hat{h} \bigr) \Big].
\end{equation}
Hence, substituting Eq. \eqref{oo1} and Eq. \eqref{oo2} into Eq. \eqref{iii}, since $h^{(1)} =  0$, we get
\begin{equation*}
\Exp ~ \hat \E_n(\bar h_T) - \hat \E_n(\hat h)
\le \frac{1}{2} \Bigl( \frac{1}{\gamma T} \big \| \hat h \big \|^2 + 
2 L^2 \rx^2 \Bigl( 1 + \frac{4 \bigl( {\rm{log}}(n) + 1 \bigr)}{n} \Bigr) \gamma \Bigr)
+ \frac{2 \rx^2 L^2 \bigl( {\rm{log}}(n) + 1 \bigr)}{\la n}.
\end{equation*}
The desired statement follows from optimizing the above bound with respect to $\gamma$.
\end{proof}

\section{Analysis for ERM Algorithm}
\label{ERM_analysis}

In this section, we repeat the statistical study described in the paper for
the family of ERM algorithms introduced in Eq. \eqref{RERM}. We obtain excess transfer risk 
bounds which are equivalent, up to constants and logarithmic factors, to those given in the paper 
for the SGD family.

We start from reminding the definition of the biased ERM algorithm in Eq. \eqref{RERM}
\begin{equation} \label{RERM_app}
\whD(\Zn) = \argmin_{w \in \Real^d}~ \cR_{\Zn,h}(w),
\end{equation}
where, for any $w, h \in \Real^d$, we recall the notation used for the empirical 
error and its biased regularized version
\begin{equation}
\begin{split}
\cR_{\Zn}(w) & = \frac{1}{n} \sum_{k = 1}^n \ell_k \bigl( \langle x_k, w \rangle \bigr) \\
\cR_{\Zn,h}(w) & = \cR_{\Zn}(w) + \frac{\la}{2} \| w - h \|^2.
\end{split}
\end{equation}

In this case, we assume to have an oracle providing us with this exact 
estimator and we ignore how much it costs. The study proceeds as
in the paper, for the SGD family. The main difference relies on using
in the decompositions, instead of \shortautoref{estimation_error},  
the following standard result on the generalization error of ERM
algorithm.

\begin{proposition} \label{generalization_error}
Let \shortautoref{bounded_inputs} and \shortautoref{lipschitz_loss} hold. Let $\whD$ the ERM
algorithm in Eq. \eqref{RERM}. Then, for any $h \in \H$, we have that 
\begin{equation}
\begin{split}
\Exp_{\Zn \sim \task^n} \big[ \cR_\task \bigl( \wh(\Zn) \bigr) - 
\cR_{\Zn} \bigl( \wh(\Zn) \bigr) \big] \le \frac{L^2 \rx^2}{\lambda n}.
\end{split}
\end{equation}
\end{proposition}

In order to prove \shortautoref{generalization_error}, we recall the following 
standard result linking the generalization error of the algorithm with its stability. 
We refer to \cite{bousquet2002stability} for more details. 

\begin{lemma}[See {\cite[Lemma $7$]{bousquet2002stability}}] \label{stability_and_generalization}
Let $A(\Zn)$ be a (replace-one) uniformly stable algorithm with parameter $\beta_n$. Then,
we have that
\begin{equation}
\Exp_{\Zn \sim \task} \big[ \cR_\task(A(\Zn)) - \cR_{\Zn} (A(\Zn)) \big] \le \beta_n.
\end{equation}
\end{lemma}

We now are ready to present the proof of \shortautoref{generalization_error}.

\begin{proof}{\bf of \shortautoref{generalization_error}}
The proof of the statement proceeds by stability arguments. Specifically, we show that,
for any $h$, $w_h(\Zn)$ is (replace-one) uniformly stable with parameter $\beta_n$ 
satisfying $\displaystyle \beta_n \le L^2 \rx^2/(\la n)$. 
Denote by $\Zn^i$ the dataset $\Zn$ in which we change the point $z_i$ with another 
independent point sample fro the same task $\task$. Thanks to 
\shortautoref{bounded_inputs} and \shortautoref{lipschitz_loss}, we have that
\begin{equation} \label{minni}
\sup_{i} \Big | \ell_i(  \langle x_i , w_h(\Zn)\rangle ) - \ell_i( \langle x_i , w_h(\Zn^i)\rangle ) \Big |
\le L \rx \big \| w_h(\Zn) - w_h(\Zn^i) \big \|.
\end{equation}
Now thanks to the $\la$-strong convexity of $\cR_{\Zn,h}$, and the definition of the algorithm, we have that
\begin{equation}
\begin{split}
& \frac{\la}{2} \| w_h(\Zn^i) - w_h(\Zn) \big \|^2 \le \cR_{\Zn,h}(w_h(\Zn^i)) - \cR_{\Zn,h}(w_h(\Zn)) \\
& \frac{\la}{2} \| w_h(\Zn) - w_h(\Zn^i) \big \|^2 \le \cR_{\Zn^i,h}(w_h(\Zn)) - \cR_{\Zn^i,h}(w_h(\Zn^i)).
\end{split}
\end{equation}
Hence, summing these two inequalities, observing that
\begin{equation*}
\hspace{-1cm}
\cR_{\Zn,h}(w_h(\Zn^i)) - \cR_{\Zn,h}(w_h(\Zn)) + \cR_{\Zn^i,h}(w_h(\Zn)) - \cR_{\Zn^i,h}(w_h(\Zn^i))
\le \frac{1}{n} \sup_{i} \Big | \ell_i( \langle x_i , w_h(\Zn)\rangle) - \ell_i(\langle x_i , w_h(\Zn^i)\rangle) \Big |,
\end{equation*}
and using again \shortautoref{bounded_inputs} and \shortautoref{lipschitz_loss}, 
we can write
\begin{equation}
\lambda \| w_h(\Zn^i) - w_h(\Zn) \big \|^2 \le \frac{1}{n} \sup_{i} 
\Big | \ell_i( \langle x_i , w_h(\Zn)\rangle) - \ell_i(\langle x_i , w_h(\Zn^i)\rangle) \Big | \\
\le \frac{L \rx}{n} \big \| w_h(\Zn) - w_h(\Zn^i) \big \|.
\end{equation}
Hence, we get
\begin{equation}
\big \| w_h(\Zn^i) - w_h(\Zn) \big \| \le \frac{L \rx}{\la n}.
\end{equation}
Therefore, continuing with Eq. \eqref{minni}, we get
\begin{equation}
\sup_{i} \Big | \ell_i(  \langle x_i , w_h(\Zn) \rangle ) - \ell_i( \langle x_i, w_h(\Zn^i)\rangle ) \Big |
\le \frac{L^2 \rx^2}{\la n}.
\end{equation}
The statement follows by applying \shortautoref{stability_and_generalization}.
\end{proof}

We now are ready to proceed with the statistical analysis of the biased ERM algorithm.
In the following \shortautorefsubapp{fixed_bias_ERM} we report the analysis for a 
fixed bias, while in \shortautorefsubapp{estimated_bias_ERM} we focus on the bias 
returned by running \shortautoref{OGDA2_paper}. 


\subsection{Analysis for a Fixed Bias}
\label{fixed_bias_ERM}

Here we study the performance of a fixed bias $h$. The following theorem should
be compared with \shortautoref{excess_transfer_risk_fixed_h} in the paper.

\begin{theorem}[Excess Transfer Risk Bound for a Fixed Bias $h$, ERM] \label{excess_transfer_risk_fixed_h_REM}
Let \shortautoref{bounded_inputs} and \shortautoref{lipschitz_loss} hold.
Let $\whD$ the biased ERM algorithm in Eq. \eqref{RERM} with regularization parameter
\begin{equation}
\la = \frac{\rx L}{{\rm Var}_h}~\sqrt{\frac{1}{n}}. 
\end{equation}
Then, the following bound holds
\begin{equation}
\een(\whD) - \eeinf \le {\rm Var}_h~2 \rx L~\sqrt{\frac{1}{n}}.
\end{equation}
\end{theorem}

\begin{proof} 
For $\task \sim \env$, consider the following decomposition 
\begin{equation} \label{dec}
\Exp_{\Zn \sim \task^n}~\cR_\task(\whD(\Zn)) - \cR_\task(\wmu) \le \text{A} + \text{B},
\end{equation}
where, A and B are respectively defined by
\begin{equation} \label{est}
\begin{split}
A & = \Exp_{\Zn \sim \task^n}~\big[\cR_{\task}(\whD(\Zn)) - \cR_{\Zn}(\whD(\Zn)) \big] \\
B & = \Exp_{\Zn \sim \task^n}~\big[\cR_{\Zn,h}(\whD(\Zn)) - \cR_{\task}(w_{\task}) \big].
\end{split}
\end{equation}
In order to bound the term A, we use \shortautoref{generalization_error}.
As regards the term B, we apply Eq. \eqref{approx} in the paper.
The desired statement derives from combining the bounds on the two terms, taking 
the average of the result with respect to $\task \sim \env$ and optimizing with respect to 
$\la$ the entire bound.
\end{proof}


\subsection{Analysis for the Bias $\bar h_T$ Returned by \shortautoref{OGDA2_paper}}
\label{estimated_bias_ERM}

In this part we study the performance of the bias $\bar h_T$ returned by an exact version of \shortautoref{OGDA2_paper}.
As a matter of fact, in this case, differently from the case analyzed in the paper for the SGD family, thanks to the assumption 
on the availability of the ERM algorithm in exact form and the closed form of the gradient of the meta-objective $\LL_{\Zn}$ (see \shortautoref{properties}), \shortautoref{OGDA2_paper} is assumed to run with exact meta-gradients. 
The following theorem should be compared with \shortautoref{excess_transfer_risk_LTL} in the paper.

\begin{theorem}[Excess Transfer Risk Bound for the Bias $\bar h_T$ Estimated by \shortautoref{OGDA2_paper}, ERM]
\label{excess_transfer_risk_LTL_REM}
Let \shortautoref{bounded_inputs} and \shortautoref{lipschitz_loss} hold. 
Let $\bar h_T$ be the output of \shortautoref{OGDA2_paper} with
exact meta-gradients and 
\begin{equation}
\gamma = \displaystyle \frac{\| \wbar \|}{L \rx}~\sqrt{\frac{1}{T}}.
\end{equation}
Consider $w_{\bar h_T}$ the biased ERM algorithm in Eq. \eqref{RERM} with bias $h = \bar h_T$
and regularization parameter 
\begin{equation}
\la = \frac{\rx L}{{\rm Var}_{\wbar}}~\sqrt{\frac{1}{n}}.
\end{equation}
Then, the following bound holds
\begin{equation}
\Exp ~ \een(w_{\bar h_T}) - \eeinf \le {\rm Var}_{\wbar} ~ 2 \rx L  
~\sqrt{\frac{1}{n}} + \| \wbar \| ~ L \rx~\sqrt{\frac{1}{T}},
\end{equation}
where the expectation above is with respect to the sampling 
of the datasets $\Zn^{(1)}, \dots, \Zn^{(T)}$ from the environment $\env$.
\end{theorem}

\begin{proof} 
We consider the following decomposition
\begin{equation}
\Exp ~ \een(w_{\bar h_T}) - \eeinf \le \text{A} + \text{B} + \text{C},
\end{equation}
where we have defined the following terms
\begin{equation}
\begin{split}
\text{A} & = \een(w_{\bar h_T}) - \hat \E_n(\bar h_T) \\
\text{B} & = \Exp ~ \hat \E_n(\bar h_T) -  \hat \E_n(\wbar) \\
\text{C} & = \hat \E_n(\wbar) - \eeinf.
\end{split}
\end{equation}
Now, in order to bound the term A, we use \shortautoref{generalization_error}
with $h = \bar h_T$ and we average with respect to $\task \sim \env$. As regards the 
term C, we apply the inequality given in Eq. \eqref{approx} with $h = \wbar$
and we take again the average on $\task \sim \env$. Finally, the term B is the convergence rate of 
\shortautoref{OGDA2_paper}, but this time, with exact meta-gradients. Now, repeating
exactly the same steps described in the proof \shortautoref{regret_OGDA_proj} with 
$\hat h = \wbar$ and $\epsilon_{n,\la}= 0$, it is immediate to show that for the choice 
of $\gamma$ given in the statement we have that
\begin{equation*}
\text{B} = \Exp ~ \hat \E_n(\bar h_T) - \hat \E_n(\wbar)
\le \big \| \wbar \big \| ~~ L \rx~\sqrt{\frac{1}{T}}
+ \frac{2 \rx^2 L^2}{\la n}.
\end{equation*}
The desired statement follows from combining the bounds on the three terms and 
optimizing the above bound with respect to $\la$.
\end{proof}

Looking at the results above, we immediately see that, up to constants and logarithmic factors, 
the LTL bounds we have stated in the paper for the SGD family are equivalent to the ones 
we have reported in this appendix for the biased ERM family.


\section{Hyper-parameters Tuning in the LTL Setting}
\label{validation_app}

\newcommand{\mdata}{{\bf Z}}

Denote by ${\bar h}_{T, \la, \gamma}$ the output of \shortautoref{OGDA2_paper} computed with $T$ iterations (hence $T$ tasks)
with values $\la$ and $\gamma$. In all experiments, we obtain this estimator $\bar h_{T_{\rm tr},\la, \gamma}$ by learning it 
on a dataset $\mdata_{\rm tr}$ of $T_{\rm tr}$ {\em training} tasks, each comprising a dataset $\Zn$ of $n$ input-output 
pairs $(x,y)\in\X\times\Y$. We perform this meta-training for different values of $\la\in \{\lambda_1,\dots,\lambda_p\}$ and 
$\gamma \in \{\gamma_1,\dots,\gamma_r\}$ and we select the best estimator based on the prediction error measured on a 
separate set $\mdata_{\rm va}$ of $T_{\rm va}$ {\em validation} tasks. Once such optimal $\la$ and $\gamma$ values have 
been selected, we report the average risk of the corresponding estimator on a set $\mdata_{\rm te}$ of 
$T_{\rm te}$ {\em test} tasks. 

In particular, for the synthetic data we considered 10 (30 for the real data) candidates values for both $\lambda$ and $\gamma$ in the range $[10^{-6}, 10^3]$ ($[10^{-3}, 10^3]$ for the real data) with logarithmic spacing.

Note that the tasks in the test and validation sets $\mdata_{\rm te}$ and $\mdata_{\rm va}$ 
are all provided with both a training and test dataset both sampled from the same distribution. Since we are 
interested in measuring the performance of the algorithm trained with $n$ points, the training datasets have all the same sample 
size $n$ as those in the meta-training datasets in $\mdata_{\rm tr}$, while the test datasets contain $n'$ points each, for some positive integer $n'$. Indeed, in order to evaluate the performance of a bias $h$, we need to first train the corresponding
algorithm $\bwhD$ on the training dataset $Z_n$, and then test its performance on the test set $Z'_{n'}$, by  computing the
empirical risk $\cR_{Z_{n'}}(\bwhD(\Zn))$. 

In addition to this, since we are considering the online setting, the training datasets arrive one at the time, therefore model selection is performed {\em online}: the system keeps track of all candidate values $\bar h_{T_{\rm tr},\la_j,\gamma_k}$, $j \in [p]$, $k \in [r]$, and, whenever a new training task is presented, these vectors are all updated by incorporating the corresponding new observations. The best bias $h$ is then returned at each iteration, based on its performance on the validation set $\mdata_{\rm va}$. The previous procedure describes how to tune simultaneously both $\la$ and $\gamma$. When the bias $h$ we use is fixed a priori (e.g. in ITL), we just need to tune the parameter $\la$; in such a case the procedure is analogous to that described above. 



\section{Additional Experiments}
\label{additional_exps}

Our method uses SGD (\shortautoref{Within-Task Algorithm Online_paper}) in 
two ways (i) to estimate the meta-gradient during meta-training and (ii) to 
evaluate the bias during the meta-validation or testing phase. In this section, we report additional experiments, in which we compared the proposed approach with exact meta-gradient approaches based on ERM. In the following experiments we approximate the ERM algorithm by running FISTA algorithm (see \cite{beck2009fast}) up to convergence 
on the within-taks dual problem introduced in \shortautoref{primal_dual_formulation_app},
see \shortautorefsubapp{approx_ERM} below for more details.

In particular, we evaluated the following three settings.

\begin{itemize}
    \item LTL SGD-SGD (our LTL method described in the paper): we use SGD both during meta-training and meta-validation / testing phases.
    \item LTL ERM-SGD: we use exact meta-gradients (computed by the ERM, as described in \shortautoref{properties} in the text) during the meta-training phase, but we apply SGD during the meta-validation/testing.
    \item LTL ERM-ERM: we use ERM both for meta-training process (to compute the exact meta-gradients) and during meta-validation/testing. This is the approach we theoretically analyzed in \shortautoref{ERM_analysis}.
\end{itemize}
We also compare the above method with four ITL settings:
\begin{itemize}
    \item ITL ERM: we perform independent task learning using the ERM algorithm with bias $h = 0$.
    \item ITL SGD: we perform independent task learning using the SGD algorithm with bias $h = 0$.
    \item MEAN ERM: we perform independent task learning using the ERM algorithm with bias $h = \wbar$ 
(only in synthetic experiments, in which this quantity is available). 
    \item MEAN SGD: we perform independent task learning using the SGD algorithm with bias $h = \wbar$ 
(only in synthetic experiments in which this quantity is available).
\end{itemize}

\begin{figure}[t] 
\begin{center}
    \includegraphics[width=0.48\textwidth]{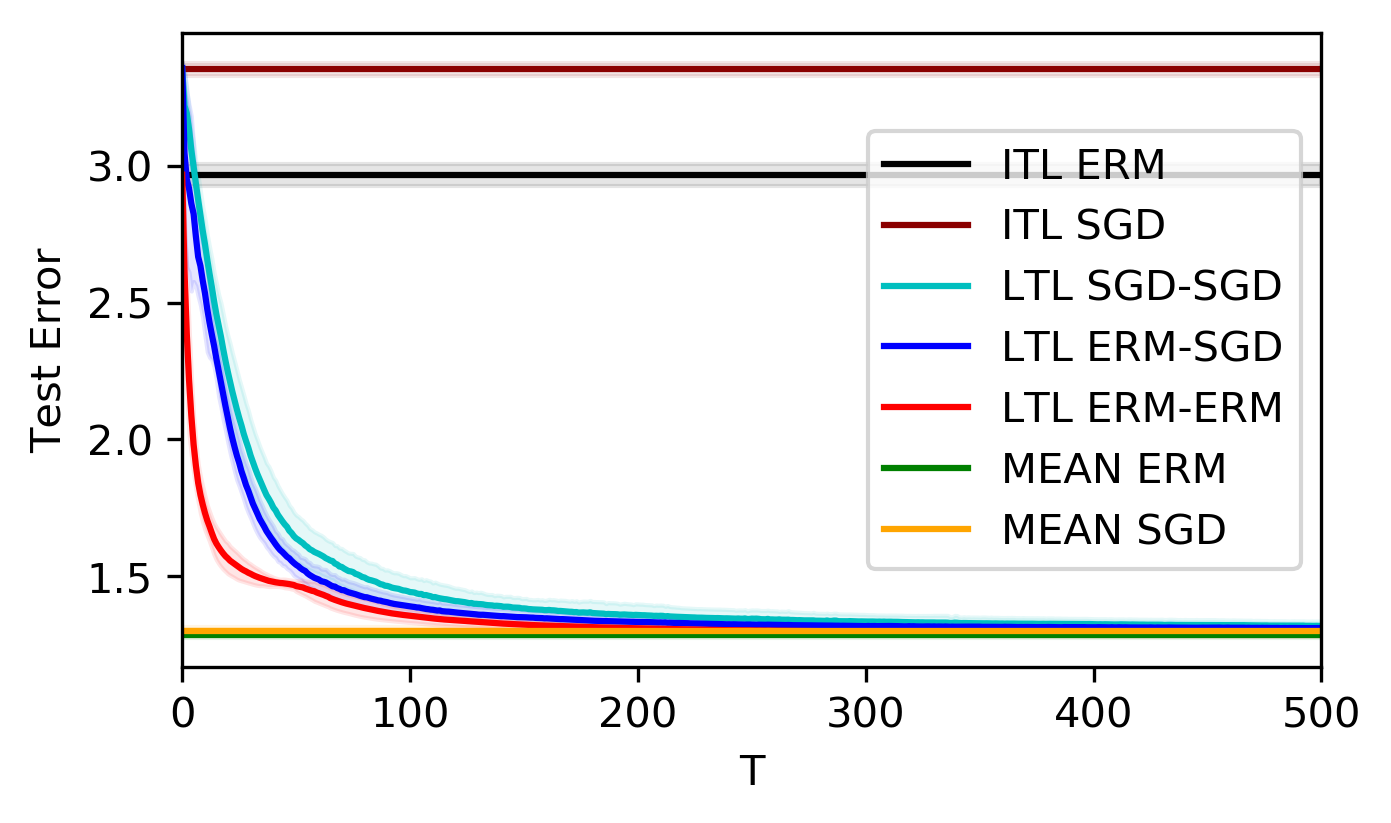} 
\quad 
    \includegraphics[width=0.48\textwidth]{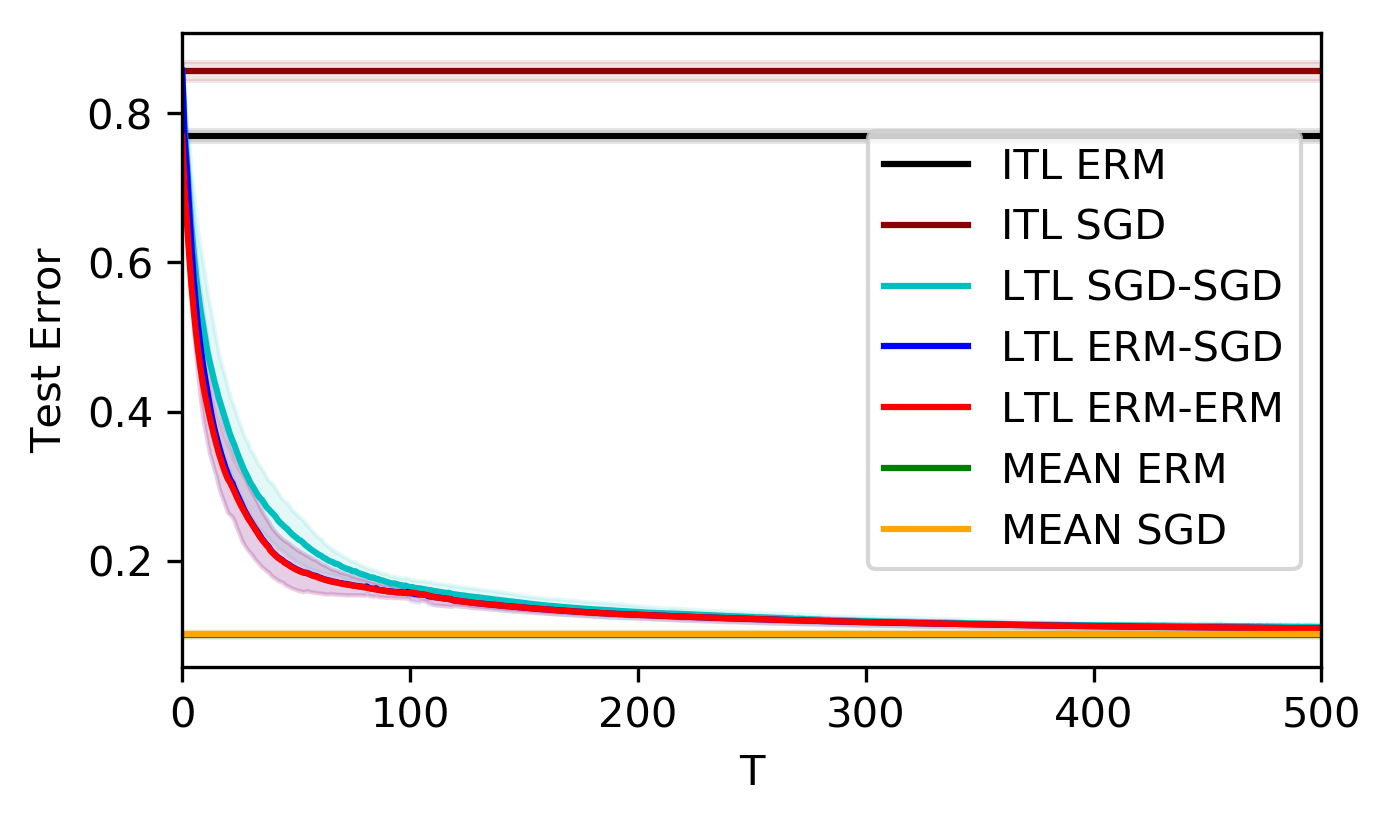}
    \caption{{\bf Synthetic Data.} Test performance of different bias with respect to an increasing number 
of tasks. (Left) Regression with absolute loss. (Right) Classification with hinge loss. The results are averaged 
over $10$ independent runs (datasets generations).\label{fig:synth-data-app}}
\end{center}
\end{figure}

\begin{figure}[t] 
\begin{center}
    \includegraphics[width=0.48\textwidth]{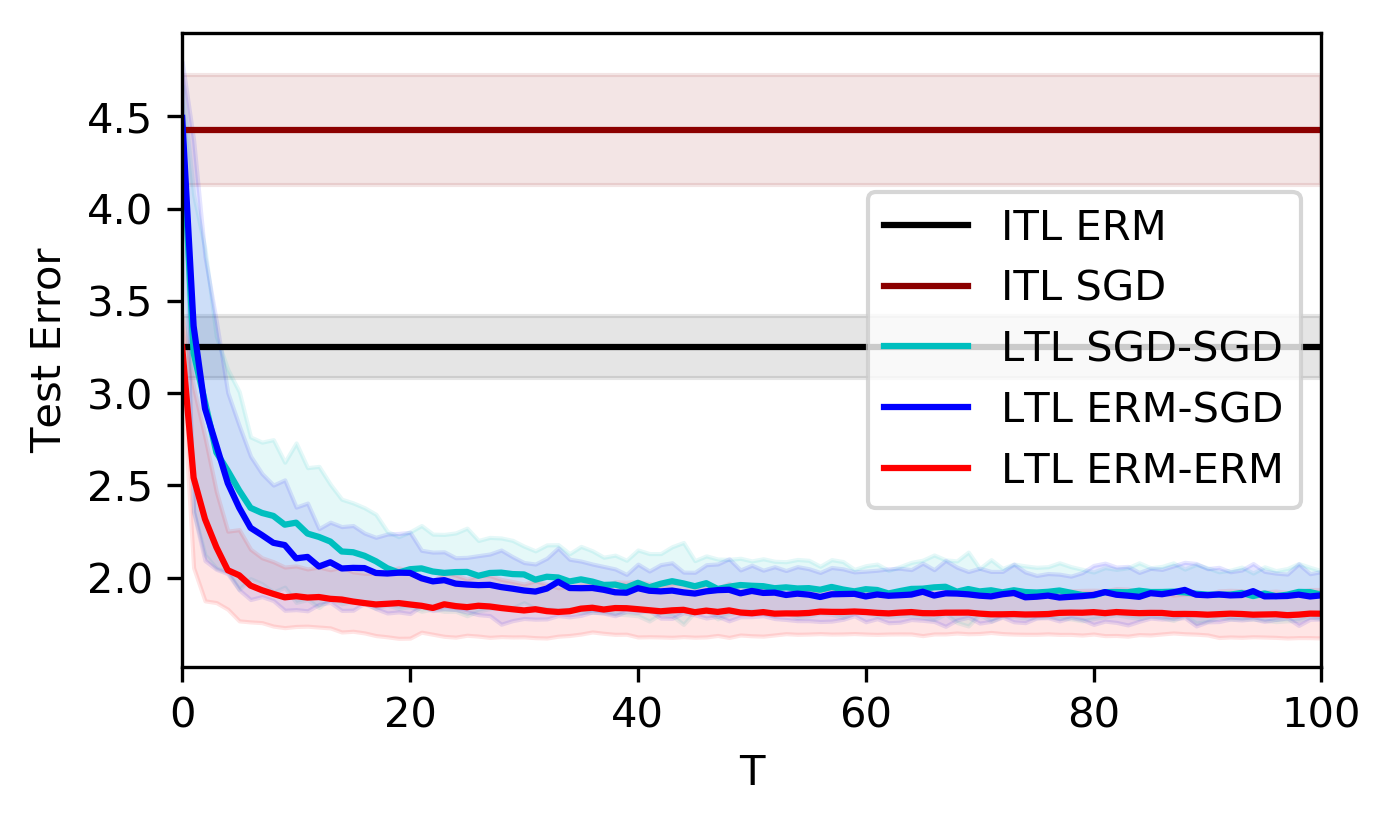}
\quad 
    \includegraphics[width=0.48\textwidth]{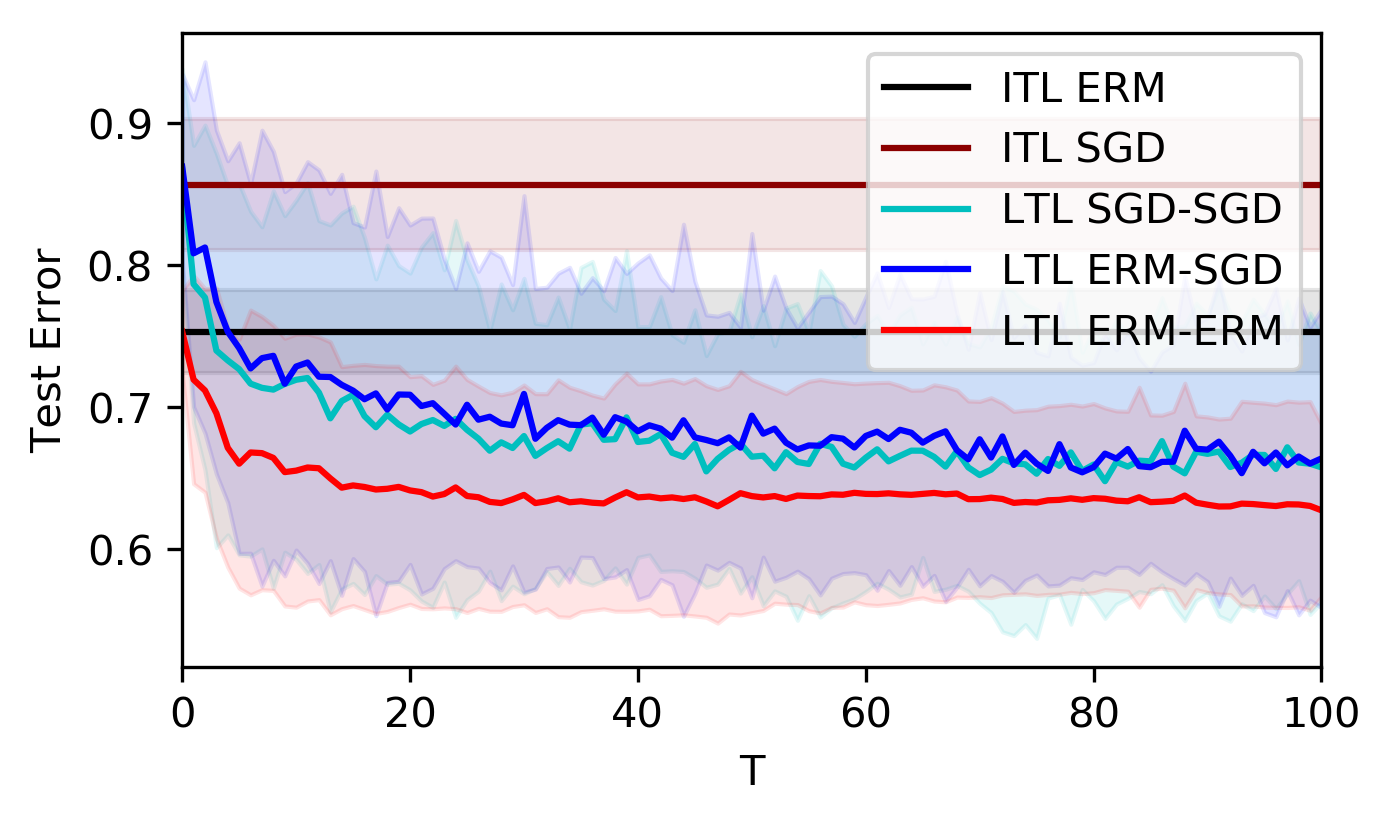}
    \caption{{\bf Real Data.} Test performance of different bias with respect to an increasing number of tasks. (Left) 
Lenk Dataset Regression. (Right) Lenk Dataset Classification. The results are averaged over $30$ independent runs (datasets generations).\label{fig:real-data-app}}
\end{center}
\end{figure}

We evaluated the performance of all the settings described above in the synthetic and real datasets used in the paper in \shortautoref{exps_sec}. The results are reported in Fig.~\ref{fig:synth-data-app} and Fig.~\ref{fig:real-data-app}, respectively. Looking at the plots, we can observe that, in all the experiments, SGD-SGD and ERM-SGD perform similarly. This confirms our theoretical finding:
approximating the meta-gradients by SGD introduces an error which does not significanlty affect the 
resulting generalization performance,
and, at the same time, it allows us to obtain an overall method with a very low computational cost.

We also point out that ERM-ERM achieves lower loss values than the other two LTL methods but, especially on the synthetic experiments, the difference is almost negligible and this is
coherent with the results obtained in \shortautoref{ERM_analysis}.
Finally, as already observed in the paper, all the LTL methods perform better than the ITL approaches (ITL ERM and ITL SGD) by a large margin, and, as expected, in the synthetic experiments, they almost match the performance of both MEAN ERM and MEAN SGD when the number of training tasks $T$ is 
sufficiently large.



\subsection{Approximating ERM by FISTA}
\label{approx_ERM}

In this section we describe how we apply FISTA algorithm (\cite{beck2009fast})
on the dual within-task problem in order to compute an approximation of the ERM 
algorithm in Eq. \eqref{RERM}.

We start from recalling the the primal within-task problem 
\begin{equation} \label{primal_problem}
\whD = \argmin_ {w \in \Real^d} \prhD(w) \quad \quad  
\prhD(w) = \frac{1}{n} \sum_{i=1}^n \ell_i\bigl( \langle x_i, w \rangle \bigr) ~ 
+ ~ \frac{\la}{2} ~ \| w - h \|^2
\end{equation}
and we rewrite its dual as follows
\begin{equation}
\uhD \in \argmin_ {u \in \Real^n} \duhD(u) \quad \quad 
\duhD(u) = G(u) + F_h(u)
\end{equation}
\begin{equation}
G(u) = \frac{1}{n} \sum_{i =1}^n \ell_i^*(n u_i) \quad \quad 
F_h(u) = \frac{1}{2 \la} \big \| \Xn \trans u \big \|^2 - \big \langle \Xn h, u \big \rangle.
\end{equation}
We apply FISTA algorithm (\cite{beck2009fast}) on this function
$\duhD$, treating $F_h$ as the smooth part and $G$ as the non-smooth
proximable part. The primal variable is then defined as before from the 
dual one by the KKT conditions. The algorithm is reported in 
\shortautoref{Within-Task Algorithm Batch With Acceleration paper} below.
In the experiments reported above, we run 
\shortautoref{Within-Task Algorithm Batch With Acceleration paper}
for $K = 2000$ iterations or until the duality gap 
\begin{equation}
\prhD \bigl( \uhD^{(k)} \bigr) + \duhD \bigl( \whD^{(k)} \bigr) 
\end{equation}
is lower than $10^{-6}$.

\begin{algorithm}[H]
\caption{Approximation of ERM by FISTA Algorithm}
\label{Within-Task Algorithm Batch With Acceleration paper}
\begin{algorithmic}
\State ~
   \State {\bfseries Input} ~~ $K$ number of iterations, $\gamma = \la/ (n \rx^2)$ step size, $\lambda > 0$, $h \in \Real^d$,
   $t_1 = 1$
   \vspace{.2cm}
   \State {\bfseries Initialization} ~~ $\uhD^{(0)} = p_h^{(1)} \in \Real^n$
  \vspace{.2cm}
   \State {\bfseries For} ~~ $k=1$ to $K$
    \vspace{.1cm}
   \State \qquad Update ~~ $\uhD^{(k)} = \prox_{\gamma G} \Bigl( p_h^{(k)} - 
   \gamma \nabla F_h \bigl(p_h^{(k)} \bigr)\Bigr)$
    \vspace{.1cm}
   \State \qquad Define ~~ $w_h^{(k)} = \displaystyle - \frac{1}{\la} \Xn \trans \uhD^{(k)} + h$ \quad \quad \text{KKT condition}
     \vspace{.1cm}
   \State \qquad Update ~~ $t_{k+1} = \displaystyle \frac{1 + \sqrt{1 + 4 t_k^2}}{2}$
    \vspace{.1cm}
   \State \qquad Update ~~ $p_h^{(k+1)} = \uhD^{(k)} + \displaystyle \frac{t_k - 1}{t_{k+1}} \bigl(\uhD^{(k)} - \uhD^{(k-1)} \bigr)$
   \vspace{.2cm}
 \State {{\bfseries Return} ~~ $w_h^{(K)} \approx \whD$} 
\State ~
\end{algorithmic}
\end{algorithm}

More precisely, we observe that, thanks to \shortautoref{bounded_inputs}, 
for any $h \in \H$, $F_h$ is $\bigl( n \rx^2 /\la \bigr)$-smooth. As a matter of fact, 
for any $u \in \Real^n$, its gradient is given by
\begin{equation}
\nabla F_h(u) = \frac{1}{\la} \Xn \Xn \trans u - \Xn h
\end{equation}
and $\big \| \Xn \Xn \trans \big \|_\infty \le n \rx^2$. The term $G$
play the role of the non-smooth part and, thanks to its separability,
for any step-size $\gamma > 0$, any $i \in [n]$ and any $u \in \Real^n$, 
we have 
\begin{equation}
\Bigl( \prox_{\gamma G} (u) \Bigr)_i = \frac{1}{n} ~ \prox_{n \gamma \ell_i^*} (n u_i).
\end{equation}

Note that, by Moreau's Identity \cite[Thm.~14.3]{bauschke2011convex},
for any $\eta > 0$ and any $a \in \Real$,  we have $\prox_{\eta \ell_i^*}(a) = 
a - \eta \prox_{\frac{1}{\eta} \ell_i} \bigl( a/\eta \bigr)$.
At last, we report the conjugate, the subdifferential and the closed form of the proximity 
operator for the absolute and the hinge loss used in our experiments.

\begin{example}[Absolute Loss for Regression and Binary Classification] \label{absolute_loss}
Let $\Y \subseteq \Real$ or $\Y = \{ \pm 1 \}$. For any $\hat y, y \in \Y$, let 
$\ell(\hat y, y) = \big | \hat y - y \big |$ and denote $\ell_y(\cdot) = \ell(\cdot, y)$.
Then, we have
\begin{equation}
u \in \partial \ell_y(\hat y) \iff u \in 
\begin{cases}
\{ 1 \} & \text{ if } \hat y - y > 0 \\
\{ - 1 \} & \text{ if } \hat y - y < 0 \\
[-1, 1] & \text{ if } \hat y - y = 0.
\end{cases}
\end{equation}
Moreover, for any $y \in \Y$, $\ell_y(\cdot)$ is $1$-Lipschitz, and, for any $u \in \Real$, $\eta > 0$, $a \in \Real$, we have that
\begin{equation}
\ell_y^*(u) = \iota_{[-1,1]}(u) + \langle u, y \rangle
\end{equation}
\begin{equation}
\prox_{\frac{1}{\eta} \ell_y}(a) =
\begin{cases}
a - \frac{1}{\eta} & \text{if $a - y > \frac{1}{\eta}$} \\
y & \text{if $a - y \in \Big[ -\frac{1}{\eta}, \frac{1}{\eta} \Big ]$} \\
a + \frac{1}{\eta} & \text{if $a - y < - \frac{1}{\eta}$}.
\end{cases}
\end{equation}
\end{example}

\begin{example}[Hinge Loss for Binary Classification] \label{hinge_loss}
Let $\Y = \{ \pm 1 \}$. For any $\hat y, y \in \Y$, let 
$\ell(\hat y, y) = \max \big \{ 0, 1 - y \hat y \big \}$ and
denote $\ell_y(\cdot) = \ell(\cdot, y)$. Then, we have
\begin{equation}
u \in \partial \ell_y(\hat y) \iff u \in 
\begin{cases}
\{ - y \} & \text{ if } 1 - y \hat y > 0 \\
\{ 0 \} & \text{ if } 1 - y \hat y < 0 \\
[-1, 1] \{- y \} & \text{ if } 1 - y \hat y = 0.
\end{cases}
\end{equation}
Moreover, for any $y \in \Y$, $\ell_y(\cdot)$ is $1$-Lipschitz, and, for any $u \in \Real$, $\eta > 0$, $a \in \Real$, we have that
\begin{equation}
\ell_i^*(u) = \frac{u}{y} + \iota_{[-1,0]} \Bigl( \frac{u}{y} \Bigr)
\end{equation}
\begin{equation}
\prox_{\frac{1}{\eta} \ell_y}(a) =
\begin{cases}
a + \frac{y}{\eta} & \text{if $ya < 1 - \frac{y^2}{\eta}$} \\
\frac{1}{y} & \text{if $ya \in \Big [ 1 - \frac{y^2}{\eta}, 1 \Big]$} \\
a & \text{if $ya > 1$}.
\end{cases}
\end{equation}
\end{example}

\end{appendices}


\end{document}